\documentclass[12pt]{article}

\usepackage[left=2.7cm,bottom=2.7cm,right=2.7cm,top=2.7cm]{geometry}

\usepackage{enumerate}
\usepackage{tabularx}
\usepackage{subfigure}
\usepackage{multirow}
\usepackage{algorithm}
\usepackage{authblk}
\usepackage{algorithmic}
\usepackage{indentfirst}
\usepackage{setspace}
\usepackage{color}
\usepackage[utf8]{inputenc} 
\usepackage[T1]{fontenc}    
\usepackage{hyperref}       
\usepackage{url}            
\usepackage{booktabs}       
\usepackage{amsfonts}       
\usepackage{nicefrac}       
\usepackage{microtype}      
\usepackage{lipsum}		
\usepackage{graphicx}
\usepackage{natbib}
\usepackage{doi}
\usepackage[namelimits]{amsmath} 
\usepackage{amssymb}            
\usepackage{amsfonts}       
\usepackage{mathrsfs} 
\usepackage[resetlabels]{multibib}
\usepackage{bm}

\newcommand{\A}{\mathcal{A}}

\newcommand{\mB}{\mathcal{B}}
\newcommand{\I}{\mathcal{I}}
\newcommand{\mS}{\mathcal{S}}
\newtheorem{theorem}{Theorem}
\newtheorem{lemma}{Lemma}
\newtheorem{proof}{Proof}
\newtheorem{corollary}{Corollary}
\providecommand{\keywords}[1]{\textit{\quad Key words }:  #1}

\newcommand{\qedsymbol}{\hfill $\blacksquare$}
\begin{document}

\title{A Splicing Approach to Best Subset of Groups Selection}


\author[1, 3]{Yanhang Zhang\textsuperscript{*}}
\author[2]{Junxian Zhu\textsuperscript{*}}
\author[3]{Jin Zhu\textsuperscript{*}}
\author[4]{Xueqin Wang\textsuperscript{$\dagger$}}

\affil[1]{\footnotesize School of Statistics,Renmin University of China}
\affil[2]{\footnotesize Saw Swee Hock School of Public Health, National University of Singapore}
\affil[3]{\footnotesize Southern China Center for Statistical Science, Department of Statistical Science, School of Mathematics, Sun Yat-Sen University}
\affil[4]{\footnotesize Department of Statistics and Finance/International Institute of Finance,
 	School of Management,
 	University of Science and Technology of China\\
 	\texttt{wangxq20@ustc.edu.cn} }

\date{}
\maketitle \sloppy

\begin{abstract}
  \noindent Best subset of groups selection (BSGS) is the process of selecting a small part of non-overlapping groups to achieve the best interpretability on the response variable.
  It has attracted increasing attention and has far-reaching applications in practice.
  However, due to the computational intractability of BSGS in high-dimensional settings, 
  developing efficient algorithms for solving BSGS remains a research hotspot.
  In this paper,
  we propose a group-splicing algorithm that iteratively detects the relevant groups and excludes the irrelevant ones.
  Moreover, coupled with  a novel group information criterion, 
  we develop an adaptive algorithm to determine the optimal model size.
  Under mild conditions, it is certifiable that our algorithm can identify the optimal subset of groups in polynomial time with high probability.
  Finally, we demonstrate the efficiency and accuracy of our methods by comparing them with several state-of-the-art algorithms on both synthetic and real-world datasets.
    
\end{abstract}


\begingroup\renewcommand\thefootnote{*}
\footnotetext{Equal contribution}
\begingroup\renewcommand\thefootnote{$\dagger$}
\footnotetext{Corresponding author}

\section{Introduction}\label{sec:intro}
Consider a linear model with $J$ non-overlapping groups:
\begin{align*}
{\bm y}=\sum\limits_{j=1}^{J}{\bm X}_{G_j}{\bm\beta}_{G_j}+{\bm \varepsilon},
\end{align*}
where ${\bm y}\in \mathbb{R}^n$ is the response variable, ${\bm X}_{G_j} \in \mathbb{R}^{n\times p_j}$ is the design matrix of the $j$th group,
${\bm\beta}_{G_j} \in \mathbb{R}^{p_j}$ are the regression coefficients of the $j$th group, and ${\bm \varepsilon} \in \mathbb{R}^n$ is the random error term.
Here, $\{G_j\}_{j=1}^J$ are group indices of $p$ predictors such that
$\cup_{j=1}^JG_j=\{1,\ldots,p\}$ and $G_i \cap G_j = \emptyset$ for all $i\neq j$.
Suppose the group size of the $j$th group is $p_j$.
We name the above linear model as group linear model, and it simplifies to an ordinary linear model when $p_1=\cdots=p_J=1$.
Group linear model is valid for analyzing variables with a certain group structure.
For instance, a categorical variable with several levels is often represented by a group of dummy variables.
Besides this, in a nonparametric additive model, a continuous component can be represented by a set of basis functions, e.g., a linear combination of spline basis functions \citep{huang2010}.
In multivariate response regression, one predictor contributing to various responses naturally possesses group structure \citep{peng2010}.
Additionally, prior knowledge can impose group structure on variables. A typical example is that
the genes belonging to the same biological pathway can be considered as a group in the genomic data analysis \citep{pan2010incorporating}.

In recent decades, high-dimensional group selection has played an essential role in various applications 
\citep{bach, zhao2009, obozinski2011, won2020group}.
One of the most natural formulations for this research is 
the best subset of groups selection (BSGS). 
In specific, BSGS minimizes the quadratic loss under a $\ell_{0,2}$ (pseudo) norm constraint for ${\bm\beta}$:
\begin{equation}\label{eq:constraint}
\min\limits_{{\bm\beta} \in \mathbb{R}^p} \frac{1}{2n}\| {\bm y}-{\bm X}{\bm\beta}\|_2^2\quad \text{s.t.}\ \|{\bm\beta}\|_{0,2}\leqslant T,
\end{equation}
where $\|{\bm\beta}\|_{0,2} = \sum_{j=1}^J I(\|{\bm\beta}_{G_j}\|_{2}\neq 0)$ in which 
$\| \cdot \|_2$ is the $\ell_2$ norm and $I(\cdot)$ is the indicator function, 
and the model size $T$ is a positive integer to be determined from data.
The $\ell_{0,2}$ constraint applies the $\ell_2$ penalty within a group and the $\ell_0$ penalty across the groups. Thus BSGS exclusively encourages sparsity at the group level, so that coefficients will either be zero or nonzero within a group.
Notably, when $p_1 = \ldots = p_J = 1$, BSGS boils down to the standard best subset selection problem, which is an NP-hard problem \citep{natarajan1995sparse}.
\citet{eldar2009robust} first considered the reconstruction of group-sparse signals by BSGS.
A naive approach to solving BSGS is exhaustively searching for all the possible combinations of $T$ groups,
and then using an information criterion, like Bayesian information criterion (BIC, \cite{schwarz1978estimating}),
for the choice of the model size $T$.
However, this approach is quite time-consuming and is intractable to tackle with high dimensionality.
To alleviate this issue, many researchers focused their attention on greedy-type methods. 
One of the representative methods is group orthogonal matching pursuit (GOMP, \cite{BOMP,ben2011near}). 
In each iteration, 
GOMP picks the group that has the strongest correlation
with the current residuals into the selected set.
Then, GOMP updates the residuals by projecting the response onto the linear subspace spanned by the selected groups.
After $T$ iterations, the selected set is output as the solution of GOMP.
Although these greedy-type methods have demonstrated promising performance in practice, 
strong assumptions are required to guarantee the estimation performance theoretically \citep{GIGA}.
Recently, several appealing works have shown that the best subset selection can be solved exactly on a large scale \citep{bertsimas2016, Bertsimas2020, hazimeh2020sparse}. 
\cite{bertsimas2021slowly} developed a slowly varying regression framework that 
can solve exactly BSGS with $\sim$ 30,000 variables in minutes.
Despite their usefulness, these works have not provided a certifiable polynomial complexity in theory.

To circumvent the computational intractability of BSGS in high-dimensional settings,
several regularization methods have been developed for group selection.
One popular approach to group selection is group Lasso (GLasso, \citet{Y2006}), a natural extension of the Lasso estimator \citep{T1996}.
Many classical works have addressed the theory of GLasso. 
\citet{10.1214/09-AOS778} introduced the concept of strong group sparsity and showed that for strong group-sparse signals, GLasso has an edge over the standard Lasso in estimation and prediction.
However, GLasso inherits similar drawbacks as the standard Lasso, e.g., the selection bias and the heavy shrinkage of large coefficients.
To remedy these drawbacks,
multiple efforts were made to develop adaptive Lasso \citep{zou2006} or extend various nonconvex penalties, such as smoothly clipped absolute deviation penalty (SCAD, \cite{F2001}) and minimax concave penalty (MCP, \cite{zhang2010}).
\citet{wang2008} proposed the adaptive GLasso and proved its oracle property.
\citet{W2010} studied the asymptotic selection and estimation properties of adaptive GLasso
when the number of groups $J$ exceeds the sample size $n$.
As for the nonconvex group penalties, \citet{W2007} extended the SCAD penalty to group selection and proved its oracle property.
\citet{H2012} proposed group MCP (GMCP). Additionally, they proved its oracle property in the high-dimensional scenario, in which $J$ is allowed to exceed $n$.

Different from these shrinkage-based methods, a practical and widely studied framework is
the Lagrangian form of problem \eqref{eq:constraint}, which is transformed into an unconstrained optimization problem:
\begin{equation}\label{eq:Lagrangian}
\min_{{\bm\beta} \in \mathbb{R}^p}\frac{1}{2n}\| {\bm y}-{\bm X}{\bm\beta}\|_2^2+\lambda\|{\bm\beta}\|_{0,2},
\end{equation}
where $\lambda > 0$ is the tuning parameter.
Some remarkable methods of solving problem \eqref{eq:Lagrangian} have been developed in recent years.
\citet{J2016} applied a primal-dual active set strategy to group selection
and conducted a theoretical analysis of the algorithm, such as a provable finite-step convergence and support recovery.
\citet{Hu2017} proposed a proximal gradient method to solve problem \eqref{eq:Lagrangian} and showed the proposed algorithm converges to a feasible point.
\citet{hazimeh2021grouped} presented a new algorithmic framework based on discrete mathematical optimization to solve problem \eqref{eq:Lagrangian}.
Furthermore, they established non-asymptotic prediction and estimation error bounds for the estimators.
However, owing to the nonconvexity of $\ell_{0,2}$ penalty, \eqref{eq:constraint} and \eqref{eq:Lagrangian} are not equivalent.
Hence $\lambda$ lacks a direct connection with $T$, so the $\ell_{0,2}$ regularization approaches cannot select an exact model size.
Empirically, tuning the values of $\lambda$ in problem \eqref{eq:Lagrangian} or other group penalties, such as GLasso, is quite time-consuming \citep{H2018}.
By contrast, problem~\eqref{eq:constraint} directly controls the exact level of model size via the choice of model size $T$.
Additionally, \citet{shen2013constrained} indicated that constrained form \eqref{eq:constraint} is more desirable over Lagrangian form \eqref{eq:Lagrangian} with respect to statistical properties of the solution, for example,
\eqref{eq:Lagrangian} requires slightly stronger assumptions than \eqref{eq:constraint} to achieve selection consistency.

In this paper, our primary aim is to design fast algorithms to obtain a high-quality solution to problem~\eqref{eq:constraint}.
Our contributions are three-fold:

\begin{itemize}
\item A finite-step convergence group-splicing (GSplicing) algorithm is proposed to solve BSGS iteratively.
In terms of group selection, the solution of GSplicing can cover the true subset of groups with high probability when the given model size is not less than the true model size.
Moreover, we establish an upper bound of the convergence rate of loss with high probability.

\item We propose a novel group information criterion (GIC) to identify an optimal model size for group selection.
By integrating GIC and GSplicing,
we develop an adaptive algorithm to determine the optimal selected set.
Theoretically, without prior information on the true model size, the adaptive algorithm can perfectly recover the true subset of groups in polynomial time with high probability.
Moreover, a heuristic strategy is equipped to accelerate the adaptive algorithm. 

\item 
We apply our proposed methods to both synthetic and real-world datasets. 
In synthetic experiments, comprehensive empirical comparisons with several state-of-the-art methods show the superiority of our methods across a variety of metrics.
Additionally, computational results for a real-world dataset demonstrate our approach that produces more accurate predictive power with fewer groups.

\end{itemize}

\subsection{Organization}
The rest of this paper is organized as follows. 
In Section~\ref{sec:methodology}, 
we detail the proposed methods for BSGS
and develop a novel information criterion to determine the optimal model size. 
Section 3 conducts the theoretical analysis of our algorithms, including the statistical and convergence properties.
Numerical experiments follow in Section \ref{numerical}, 
where we compare our methods with several state-of-the-art methods using both synthetic and real-world datasets.
Finally, we summarize our study in Section~\ref{sec:conclusion}. 
The online supplement includes technical proofs and additional theoretical results.

\subsection{Notations}
Let $\mS = \{1, \ldots, J\}$. For a given subset of groups $\A \subseteq \mS$ with size $|\A|$, 
denote $\#\{\A\} = \sum\limits_{j \in \A}p_j$.
Let ${\bm\beta}_{\A} = ({\bm\beta}_{G_j}, j\in \A) \in \mathbb{R}^{\#\{\A\}}$ and ${\bm X}_{\A} = ({\bm X}_{G_j}, j \in \A)\in \mathbb{R}^{n \times \#\{\A\}}$.
Denote $\mathrm{\bm I}_{p_j}$ as the $p_j \times p_j$ identity matrix. 
Denote the selected set $\A = \{j\in \mS:\|{\bm\beta}_{G_j}\|_2 \neq 0\}$ and the unselected set $\I = \mS\backslash \A=\A^c$. 
Let ${\bm\beta}^*$ be the true regression coefficients 
and $\A^*$ be the true subset of groups such that $\A^* = \{j \in \mS: \|{\bm\beta}^*_{G_j}\|_2 \neq 0 \}$. 
Assume the true model size $|\A^*|=s^*$.
Let $\I^* = (\A^*)^c$. 
Let the maximum group size $p_{\max}=\max\{p_j:1\leqslant j\leqslant J\}$ and the minimum group size $\ p_{\min}=\min\{p_j:1\leqslant j\leqslant J\}$.
Define $[\cdot]$ as the function that returns the nearest integer and $L({\bm\beta}) = \frac{1}{2n}\| {\bm y}-{\bm X}{\bm\beta}\|_2^2$ as the loss function. Throughout this paper, we assume to apply a groupwise orthonormalization, e.g., by a QR decomposition, to obtain $\dfrac{{\bm X}_{G_j}^\top {\bm X}_{G_j}}{n}=\mathrm{\bm I}_{p_j}$ for all $j \in \mS$.

\section{Methodology}\label{sec:methodology}
In Section~\ref{sec:gsplicing}, we first introduce an algorithm for solving BSGS .
In Section~\ref{sec:agsplicing}, we develop an adaptive algorithm to recover the true subset of groups.
\subsection{Group-Splicing algorithm}\label{sec:gsplicing}
The augmented Lagrangian form \citep{ito2013variational} of problem~\eqref{eq:constraint} is
\begin{align}\label{s2.1}
\begin{split}
\min\limits_{{\bm\beta},\bm \upsilon,\bm d \in \mathbb{R}^p}\quad &\frac{1}{2n}\| {\bm y}-{\bm X}{\bm\beta}\|_2^2+\bm d^\top({\bm\beta}-\bm \upsilon)+\frac{\kappa}{2}\|{\bm\beta}-\bm \upsilon\|_2^2 \\
  \text{s.t.}\quad &\|\bm \upsilon\|_{0,2} \leqslant T, \\
\end{split}
\end{align}
where $\bm \upsilon \in \mathbb{R}^p$, $\kappa$ is a positive constant and $\bm d \in \mathbb{R}^p$ is the dual variable.
Without loss of generality, the response variable and predictors are centered around the mean and let $\|\bm \upsilon\|_{0,2}=T$.
We derive the optimal conditions of problem \eqref{s2.1} as follows:

\begin{lemma}\label{optimal}
  Suppose $({\bm\beta}^\diamond, \bm \upsilon^\diamond, \bm d^\diamond)$ is a coordinate-wise minimizer of \eqref{s2.1}.
  Denote $\A^\diamond = \{j\in \mS : I(\|\bm \upsilon_{G_j}^\diamond\|_2 \neq 0)\}$ and $\I^\diamond = (\A^\diamond)^c$. 
  Then $(\bm \upsilon^\diamond, {\bm\beta}^\diamond, \bm d^\diamond)$ and $(\A^\diamond, \I^\diamond)$ satisfy:
  \begin{align*}
      &{\bm\beta}^\diamond_{\A^\diamond} = ({\bm X}^\top_{\A^\diamond} {\bm X}_{\A^\diamond})^{-1} {\bm X}_{\A^\diamond}^\top \bm y,\ {\bm\beta}^\diamond_{\I^\diamond} = 0,\\
      &\bm d^\diamond_{\A^\diamond}=0,\ \bm d^\diamond_{\I^\diamond}={\bm X}^\top_{\I^\diamond}(y-{\bm X}{\bm\beta}^\diamond)/n,\\
      &\bm \upsilon^\diamond = {\bm\beta}^\diamond,\\
      &\A^\diamond =\{j\in \mS:\sum_{i=1}^{J}I(\|{\bm\beta}_{G_j}^\diamond+\frac{1}{\kappa}\bm d^\diamond_{G_j}\|_2^2\leqslant \|{\bm\beta}_{G_i}^\diamond+\frac{1}{\kappa}\bm d^\diamond_{G_i}\|_2^2)\leqslant T\}.
  \end{align*}
\end{lemma}

\vspace{-0.1cm}
From Lemma~\ref{optimal}, $\bm \upsilon$ plays a critical role in deciding the optimal selected set $\A^\diamond$.
In general, we approximate the optimal conditions iteratively. Let $\{\A^k, \I^k, {\bm\beta}^k, \bm d^k\}$ be the solution in the $k$th iteration. We update $\{\A^{k+1}, \I^{k+1}\}$ by
\vspace{-0.1cm}
\begin{align}\label{s2.2}
\begin{split}
&\A^{k+1} = \{j\in \mS:\sum_{i=1}^{J}I(\|{\bm\beta}_{G_j}^k+\frac{1}{\kappa}\bm d^k_{G_j}\|_2^2\leqslant \|{\bm\beta}_{G_i}^k+\frac{1}{\kappa}\bm d^k_{G_i}\|_2^2)\leqslant T\},\\
& \I^{k+1} = (\A^{k+1})^c.
\end{split}
\end{align}
Then, we update the primal variable ${\bm\beta}^{k+1}$ and the dual variable $\bm d^{k+1}$ by
\begin{align*}
  &{\bm\beta}^{k+1}_{\A^{k+1}} = ({\bm X}_{\A^{k+1}}^\top {\bm X}_{\A^{k+1}})^{-1}{\bm X}_{\A^{k+1}}^\top \bm y, \; {\bm\beta}^{k+1}_{\I^{k+1}} = 0, \\
  &\bm d^{k+1}_{\A^{k+1}} = 0, \; \bm d_{\I^{k+1}}^{k+1} = {\bm X}_{\I^{k+1}}^\top (\bm y - {\bm X}{\bm\beta}^{k+1})/n.
\end{align*}
In \eqref{s2.2}, $\kappa$ weighs the importance of $\bm d^k$ in the $k$th iteration.
It is worth noting that a large $\kappa$ (e.g., $\kappa\rightarrow+\infty$) makes a minor update on the selected set,
which controls only a small number of groups change between the selected set and the unselected set.
Conversely, a small $\kappa$ (e.g., $\kappa\rightarrow 0$) might completely change the selected set.
Motivated by this observation, we consider the update~\eqref{s2.2} as an exchange between the selected set and the unselected set, 
which we call the ``splicing'' procedure \citep{Zhu202014241}.
Therefore, we can select $\kappa$ by determining the splicing size.
We precisely characterize the idea in the next paragraph. 

Suppose the size of the exchanged subset of groups is a positive integer  $C (\leqslant |\A^k|)$. 
The smallest $C$ groups in $\A^k$ and the largest $C$ groups in $\I^k$ are defined as
\begin{align}\label{s2.5}
  \begin{split}
    \mS^k_{C,1}
    &=\{j \in \A^k:\sum_{i \in \A^k} I(\|{\bm\beta}^k_{G_j}+\frac{1}{\kappa}\bm d^k_{G_j}\|_2^2 \geqslant \|{\bm\beta}^k_{G_i}+\frac{1}{\kappa}\bm d^k_{G_i}\|_2^2)\leqslant C\}\\
    &=\{j \in \A^k:\sum_{i \in \A^k} I(\|{\bm\beta}^k_{G_j}\|_2^2 \geqslant \|{\bm\beta}^k_{G_i}\|_2^2)\leqslant C\},
  \end{split}
  \end{align}
  and
  \begin{align}\label{s2.6}
  \begin{split}
    \mS^k_{C,2}
    &=\{j \in \I^k:\sum_{i \in \I^k} I(\|{\bm\beta}^k_{G_j}+\frac{1}{\kappa}\bm d^k_{G_j}\|_2^2 \leqslant \|{\bm\beta}^k_{G_i}+\frac{1}{\kappa}\bm d^k_{G_i}\|_2^2)\leqslant C\}\\
    &=\{j \in \I^k:\sum_{i \in \I^k} I(\|\bm d^k_{G_j}\|_2^2 \leqslant \|\bm d^k_{G_i}\|_2^2)\leqslant C\},
  \end{split}
  \end{align}
where the last equation in \eqref{s2.5} follows from $\bm d^k_{\A^k} = 0$, and the last equation in \eqref{s2.6} follows from ${\bm\beta}^k_{\I^k}=0$.
According to Lemma \ref{sacrifices},
$\mS^k_{C,1}\ (\mS^k_{C,2})$ can be interpreted as the groups in $\A^k\ (\I^k)$ with the smallest (largest) contributions to the decrease of loss: 
\begin{lemma}\label{sacrifices}
\begin{itemize}
  \item [(i)] For any $j \in \A^k$, the contribution to the decrease of $L({\bm\beta}^k)$ by discarding the $j$th group is
  \begin{equation*}\label{s2.3}
    L({\bm\beta}^{\A^k\backslash j})-L({\bm\beta}^k) = \frac{1}{2}\|{\bm\beta}^k_{G_j}\|_2^2,
  \end{equation*}
  where ${\bm\beta}^{\A^k\backslash j}$ is the estimator assigning the $j$th group of ${\bm\beta}^k$ to be zero.
  \item [(ii)] For any $j \in \I^k$, the contribution to the decrease of $L({\bm\beta}^k)$ by adding the $j$th group is
  \begin{align*}\label{s2.4}
    \begin{split}
      L({\bm\beta}^k)-L({\bm\beta}^k+\bm t^k_j) =  \frac{1}{2}\|\bm d^k_{G_j}\|_2^2,
  \end{split}
  \end{align*}
  where ${\bm t}^k_j = \mathop{\mathrm{argmin}}\limits_{{\bm t}_{G_j} \neq 0}L({\bm\beta}^k+\bm t)$, $\bm d^k_{G_j} = {\bm X}_{G_j}^\top(\bm y-{\bm X}{\bm\beta}^k)/n$.
\end{itemize}
\end{lemma}
As Lemma \ref{assoc:k_c} shows, the sizes of $\mS^k_{C,1}$ and $\mS^k_{C,2}$ are related to $\kappa$. 
\begin{lemma}\label{assoc:k_c}
Assume the size of the exchanged subset of groups is $C$. 
For any positive integer $C < |\A^k|$, 
the corresponding range of $\kappa$ in the $k$th iteration is 
\begin{equation*}
  \kappa \in \left(\frac{\min_{j \in \mS^k_{C+1, 2}}\|\bm d^k_{G_j}\|_2}{\max_{i\in \mS^k_{C+1, 1}}\|{\bm\beta}^k_{G_i}\|_2}, \frac{\min_{j \in \mS^k_{C, 2}}\|\bm d^k_{G_j}\|_2}{\max_{i\in \mS^k_{C, 1}}\|{\bm\beta}^k_{G_i}\|_2}\right],
  \end{equation*}
and for $C=|\A^k|$, we have
\begin{equation*}
  \kappa \in \left(0, \frac{\min_{j \in \mS^k_{C, 2}}\|\bm d^k_{G_j}\|_2}{\max_{i\in \A^k}\|{\bm\beta}^k_{G_i}\|_2}\right].
\end{equation*}
\end{lemma}
Obviously, deciding an optimal $C$ is more efficient than tuning $\kappa$.
A natural approach to minimize the loss in problem \eqref{eq:constraint} is choosing $C$ such that loss can decrease after updating the selected set.
We refer to this as ``group splicing''.
By performing group splicing in each iteration, we obtain the group-splicing (GSplicing) algorithm, which we summarize in Algorithm~\ref{alg:gsplicing}. 

\begin{algorithm}[h]
\caption{\label{alg:gsplicing} \textbf{G}roup-\textbf{Splicing} (GSplicing) algorithm}
\begin{algorithmic}[1]
\REQUIRE ${\bm X},\ \bm y,\ \{G_j\}_{j=1}^J,\ T, \ C_{\max},\ \pi_T, \ \A^0$.
\STATE Initialize $k=0$ and solve primal variable ${\bm\beta}^k$ and dual variable $\bm d^k$
\vspace{-0.3cm}
\begin{align*}
  &{\bm\beta}_{\A^k}^k = ({\bm X}_{\A^k}^\top {\bm X}_{\A^k})^{-1}{\bm X}_{\A^k}^\top \bm y,\ {\bm\beta}_{\I^k}^k = 0,\\
  &\bm d^k_{\I^k} = {\bm X}_{\I^k}^\top(\bm y-{\bm X}{\bm\beta}^k)/n,\ \bm d_{\A^k}^k = 0.
\end{align*}
\vspace{-1cm}
\WHILE{$\A^{k+1} \neq \A^k, $}
\STATE Compute $L=\frac{1}{2n}\|{\bm y-{\bm X}{\bm\beta}^k}\|_2^2$ and update $\mS_1^k, \mS_2^k$ by
\vspace{-0.1cm}
\begin{align*}
  &\mS_1^k = \{j \in \mathcal{A}^k: \sum\limits_{i\in \mathcal{A}^k} I(\|{{\bm\beta}_{G_j}^k}\|_2^2 \geqslant \|{{\bm\beta}_{G_i}^k}\|_2^2) \leqslant C_{\max}\},\\
  &\mS_2^k = \{j \in \mathcal{I}^k: \sum\limits_{i\in \mathcal{I}^k} I(\|{\bm d_{G_j}^k}\|_2^2 \leqslant \|{\bm d_{G_i}^k}\|_2^2) \leqslant C_{\max}\}.
\end{align*}
\vspace{-0.8cm}
\FOR {$C=C_{\max},\ldots, 1, $}
\STATE Let $\tilde{\A}^k_C=(\mathcal{A}^k\backslash \mS_1^k)\cup \mS_2^k\ ,\ \tilde{\I}^k_C = (\mathcal{I}^k\backslash \mS_2^k)\cup \mS_1^k$ and solve
\vspace{-0.1cm}
\begin{align*}
  &\tilde{\bm\beta}_{\tilde{\A}^k_C}=({\bm X}_{\tilde{\A}^k_C}^\top {\bm X}_{\tilde{\A}^k_C})^{-1}{\bm X}_{\tilde{\A}^k_C}^\top \bm y,\ \tilde {\bm\beta}_{\tilde{\I}^k_C}=0,\\
  &\tilde {\bm d} = {\bm X}^\top (\bm y-{\bm X}\tilde{{\bm\beta}})/n,\ \tilde{L}=\frac{1}{2n}\|{\bm y-{\bm X}\tilde{{\bm\beta}}}\|_2^2.
\end{align*}
\vspace{-0.8cm}
\IF {$L - \tilde{L} < \pi_T, $}
\STATE Denote $(\tilde{\A}^k_C, \tilde{\I}^k_C, \tilde{{\bm\beta}}, \tilde{\bm d}) \text{ as } (\A^{k+1}, \I^{k+1}, {\bm\beta}^{k+1}, \bm d^{k+1})$ and break.
\ELSE
\STATE Update $\mS_1^k\ \text{and}\ \mS_2^k$: 
\vspace{-0.3cm}
\begin{align*}
  \mS_1^k = \mS_1^k\backslash \mathop{\mathrm{argmax}}\limits_{i \in \mS_1^k} \{\|{{\bm\beta}_{G_i}^k}\|_2^2\}, \; \mS_2^k = \mS_2^k\backslash \mathop{\mathrm{argmin}}\limits_{i \in \mS_2^k} \{\|{\bm d_{G_i}^k}\|_2^2\}.
\end{align*}
\vspace{-1cm}
\ENDIF
\ENDFOR
\ENDWHILE
\ENSURE $(\A^{k+1}, \I^{k+1}, {\bm\beta}^{k+1}, \bm d^{k+1})$.
\end{algorithmic}
\end{algorithm}
We provide details about the input parameters in Algorithm~\ref{alg:gsplicing}.
The first parameter $C_{\max}$ is a positive integer no more than $T$. It controls the maximum exchanged size in Algorithm \ref{alg:gsplicing}. 
The simulation results in Section \ref{cmax} suggest that
$C_{\max} = 2$ allows Algorithms~\ref{alg:sgsplicing} and \ref{alg:ggsplicing} to obtain high-quality solutions in less runtime.
The second parameter $\pi_T$ is a threshold related to the given model size $T$. 
It prevents redundant splicings and accelerates the convergence of Algorithm \ref{alg:gsplicing}.
According to condition (C5) in Section \ref{sec:theory}, we set $\pi_T = 0.1Tp_{\max} \log p \log (\log n)/n$.
The last one is the initial selected set $\A^0$. 
Typically, we choose the $T$ largest elements of set $\{\|{\bm X}_{G_j}^\top \bm y\|_2^2, j\in \mS\}$ as $\A^0$.
Notably, Algorithm \ref{alg:gsplicing} terminates in a finite number of iterations
since the loss decreases at least $\pi_T$ in each iteration, and the choices of the selected set with fixed model size $T$ are finite.

\subsection{Adaptive group-splicing algorithm}\label{sec:agsplicing}
It is crucial to decide the optimal model size for BSGS, which is usually unknown in practice. 
A natural idea is to take the model size $T$ as a tuning parameter and run GSplicing algorithm along a sequence of $T$.
Indeed, we can set the sequence from $T=1$ to $T=T_{\max}$, where $T_{\max}$ is the upper bound of the potential model size. 
Then, we can combine some model selection techniques, 
such as information criterion, to determine the optimal model size.
One popular choice of the information criterion is the Bayesian information criterion (BIC, \citet{schwarz1978estimating}). Recall that the BIC supported on $\cup_{j \in \hat \A}G_j$ is defined as
\begin{equation*}
  \text{BIC}(\hat \A) = n\log \left(\frac{\| {\bm y}-{\bm X}\hat{\bm\beta}\|_2^2}{n}\right) +  \#\{\hat \A\}\log n ,
  \end{equation*}
where $\hat{\bm\beta}$ is the least-squares estimator given the selected set $\hat \A$. 
However, for the high-dimensional data, BIC tends to identify a model with numerous spurious predictors because of its light penalty on the model complexity \citep{chen2008ebic}.
To adapt to the high dimensionality with a group structure, 
we propose a novel information criterion for group selection named group information criterion (GIC).
We define GIC supported on $\cup_{j \in \hat \A}G_j$ as
\begin{equation*}
\text{GIC}(\hat \A) = n\log L(\hat {\bm\beta}) +  \#\{\hat \A\}\log J\log(\log n) .
\end{equation*}
GIC considers the penalty for the number of groups $J$ as $\log J$, which is adjusted adaptively by the sample size $n$, i.e., the term $\log (\log n)$.
Meanwhile, $\log (\log n)$ diverges to infinity at a slow rate to prevent underfitting. 
Employing GIC, we design a sequential group-splicing algorithm, which is summarized in Algorithm \ref{alg:sgsplicing}.
\begin{algorithm}[htbp]
  \caption{\label{alg:sgsplicing}\textbf{S}equential \textbf{G}roup-\textbf{Splicing} (SGSplicing) algorithm}
  \begin{algorithmic}[1]
  \REQUIRE ${\bm X},\ \bm y,\ \{G_j\}_{j=1}^J,\ T_{\max}, \ C_{\max}.$
  \STATE $\hat \A_{0} = \{\}, \hat{{\bm\beta}}_0 = \mathbf{0}. $
  \FOR {$T=1,\ldots, T_{\max},$}
  \STATE $\A^0_T = \hat\A_{T-1} \cup \mathop{\mathrm{argmax}}\limits_{j \in \hat \A_{T-1}^c} \{ \|{\bm X}^\top_{G_j}(\bm y-{\bm X}\hat{\bm\beta}_{T-1})\|_2^2\}.$
  \STATE $(\hat \A_{T}, \hat \I_T, \hat{\bm\beta}_{T}, \hat {\bm d}_T) =$  GSplicing$({\bm X},\ \bm y,\ \{G_j\}_{j=1}^J,\ T,\ C_{\max},\ \pi_T,\ \A_T^0)$.
  \STATE $\text{GIC}_T =\text{GIC}(\hat \A_{T})$.
  \ENDFOR
  \STATE $T^* = \mathop{\mathrm{argmin}}\limits_{T} \{ \text{GIC}_T \}$.
  \ENSURE $(\hat \A_{T^*},\hat \I_{T^*}, \hat{\bm\beta}_{T^*},\hat {\bm d}_{T^*})$.
  \end{algorithmic}
\end{algorithm}

Algorithm~\ref{alg:sgsplicing} would miss the true model size $s$ if $T_{\max}$ is small,
while runtime would visibly increase if $T_{\max}$ is proportional to $J$. 
From condition (C6) in Section \ref{sec:theory}, we suggest $T_{\max} = [\frac{n}{p_{\min} \log p}]$.
\begin{figure}[htbp]
  \centering
  \includegraphics[scale = 0.6]{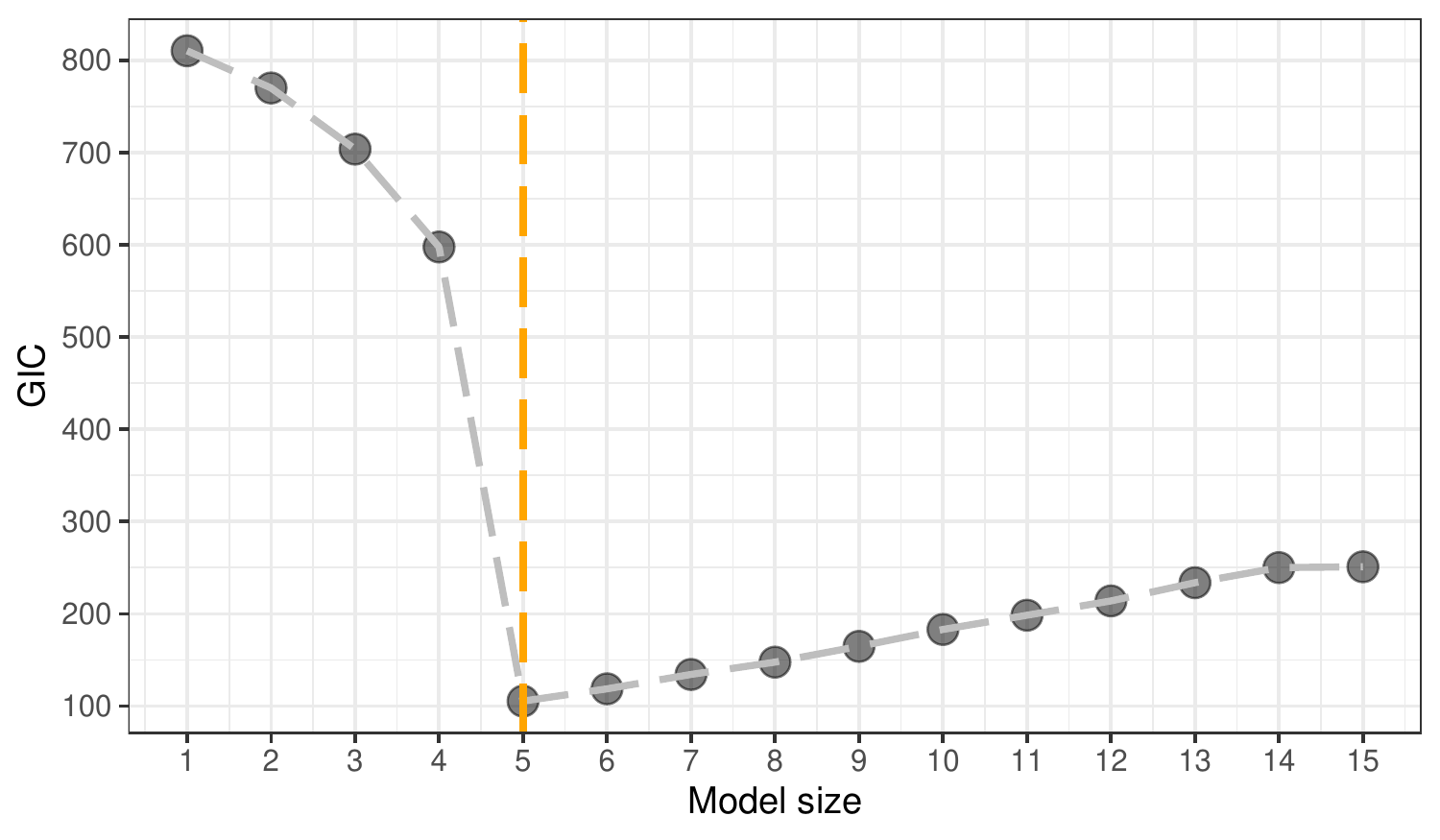}
  \caption{
    ``Model size v.s. GIC'' plot. The $x$-axis is model size,
    and the $y$-axis is GIC's value recorded in SGSplicing algorithm.
    The entries of design matrix ${\bm X}$ and error term ${\bm \varepsilon}$ are both i.i.d. sampled from $\mathcal{N}(0, 1)$. The shape of ${\bm X}$ is $200 \times 600$.
    Take the three adjacent variables as one group. Set $s^*=5$ and $T_{\max}=15$. All nonzero coefficients are equal to 2.
    The orange vertical dash line indicates the true model size.}\label{fig1}
\end{figure}

A typical GIC path of Algorithm~\ref{alg:sgsplicing} is presented in Figure~\ref{fig1}, 
from which we see that GIC decreases from $T=1$ to $T=5$, but increases as $T$ exceeds 5.
In other words, the GIC path of the SGSplicing algorithm is a strictly unimodal function achieving the minimum at the true model size $T = 5$.
Motivated by this observation, we consider a golden-section search technique \citep{k1953} to obtain the minimum of the GIC path and determine the optimal model size $T$.
We summarize the golden-section group-splicing algorithm in Algorithm~\ref{alg:ggsplicing}. 
Notably, by virtue of the golden-section search technique, Algorithm~\ref{alg:ggsplicing} runs GSplicing algorithm $O(\log_{0.618^{-1}}T_{\max})$ times rather than $O(T_{\max})$ times in Algorithm~\ref{alg:sgsplicing}. 
\begin{center}
\begin{algorithm}[htbp]
  \caption{\textbf{G}olden-section \textbf{G}roup-\textbf{Splicing} (GGSplicing) algorithm}\label{alg:ggsplicing}
  \begin{algorithmic}[1]
  \REQUIRE ${\bm X},\ \bm y,\ \{G_j\}_{j=1}^J,\ T_{\min},\ T_{\max},\ C_{\max}.$
 \STATE Initialize $T_1,\ T_2$ by $T_1 = [0.618 \times T_{\min} + 0.382 \times T_{\max}], \ T_2 = [0.382 \times T_{\min} + 0.618 \times T_{\max}]$.
 \vspace{-0.5cm}
  \STATE Run GSplicing algorithm given the model size $T_1$ and $T_2$.
  \STATE Compute GIC and denote them as $\text{GIC}_{T_1}$ and $\text{GIC}_{T_2}$.
  \WHILE {$T_1 \neq T_2,$}
  \IF {$\text{GIC}_{T_1} \leqslant \text{GIC}_{T_2},$}
  \STATE $T_{\max} = T_2,\ T_2 = T_1,\ T_1 = [0.618 \times T_{\min} + 0.382 \times T_{\max}], \text{GIC}_{T_2} = \text{GIC}_{T_1}$.
  \STATE Run GSplicing algorithm given the model size $T_1$ and compute $\text{GIC}_{T_1}$.
  \ELSE
  \STATE $T_{\min} = T_1,\ T_1 = T_2,\ T_2 = [0.382 \times T_{\min} + 0.618 \times T_{\max}], \text{GIC}_{T_1} = \text{GIC}_{T_2}$.
  \STATE Run GSplicing algorithm given the model size $T_2$ and compute $\text{GIC}_{T_2}$.
  \ENDIF
  \ENDWHILE
  \STATE $(\hat{\mathcal{A}}, \hat \I, \hat{\bm\beta}, \hat {\bm d}) =$ GSplicing(${\bm X}, \bm y, \{G_j\}_{j=1}^J, T_1, C_{\max}, \pi_{T_1}, \A_{T_1}^0$).
  \ENSURE $(\hat{\mathcal{A}}, \hat \I, \hat {\bm\beta}, \hat {\bm d})$.
  \end{algorithmic}
  \end{algorithm}
\end{center}

\section{Theoretical properties}\label{sec:theory}
In this section, we study the theoretical properties of our algorithms.
We first present the statistical properties of our algorithms in Section~\ref{subsec:statistical}. 
Next, in Section~\ref{subsec:convergence_rate}, we analyze the proposed algorithms from a computational viewpoint. 
Moreover, we provide the $\ell_2$ error bounds of the estimator in Appendix B of the online supplement.
Before formally presenting the theoretical properties, we discuss some of the technical conditions which our analysis requires.

The first condition constrains the behavior of the error term:
\begin{itemize}
  \item [(C1)] The random errors ${\bm \varepsilon}_1,\ldots,{\bm \varepsilon}_n$ are i.i.d with mean zero and sub-Gaussian tails, that is, there exists a positive number $\sigma$ such that $P(|{\bm \varepsilon}_i|>z) \leqslant 2\exp(-z^2/\sigma^2)$, for all $z \geqslant 0$.
\end{itemize} 
Condition (C1) assumes the probability distribution of error terms ${\bm \varepsilon}$ satisfies the sub-Gaussian distribution, which is slightly weaker than the standard normality.
The sub-Gaussian condition (C1) is widely used in statistical literature to analyze the non-asymptotic bounds under high-dimensional settings \citep{zhang2010, H2018, wainwright2019high}.

As a mild identifiability condition for the linear model, the sparse Riesz condition (SRC) 
is used to investigate the theory of Lasso \citep{Z2008} and MCP \citep{zhang2010}.
The design matrix ${\bm X}$ satisfies SRC with order $\tau$ and spectrum bounds $0 < c_-(\tau) \leqslant c_+(\tau) < \infty$ if
\begin{equation*}
  c_-(\tau)\|\bm u\|_2^2\leqslant\dfrac{\|{\bm X} \bm u \|_2^2}{n}\leqslant c_+(\tau)\|\bm u\|_2^2 ,\ \forall \bm u \neq 0, \bm u \in \mathbb{R}^{p}\ \text{with}\  \|\bm u\|_0 \leqslant \tau.
\end{equation*}
Recently, SRC has been extended to group selection, serving as an indispensable ingredient for the analysis of selection consistency of adaptive GLasso and GMCP \citep{W2010, H2012}.
The design matrix ${\bm X}$ satisfies group SRC (GSRC) with order $\tau$ and spectrum bounds $0 < c_*(\tau) \leqslant c^*(\tau) < \infty$ if:
\begin{equation*}
  c_*(\tau)\|\bm u\|_2^2\leqslant\dfrac{\|{\bm X}_{\A} \bm u \|_2^2}{n}\leqslant c^*(\tau)\|\bm u\|_2^2 ,\ \forall \bm u \neq 0, \bm u \in \mathbb{R}^{\#\{\A\}}\ \text{with}\  |\A| \leqslant \tau.
\end{equation*}
GSRC is a less stringent assumption compared to the standard SRC.
Suppose the design matrix ${\bm X}$ satisfies SRC with order $N$ and spectrum bounds $0 < c_{-}(N) \leqslant c_{+}(N) < \infty$, 
where $N$ is the summation of the largest $\tau$ group sizes. 
By definition, ${\bm X}$ satisfies SRC for all $N$-sparse vectors with at most $\tau$ nonzero groups. This implies ${\bm X}$ satisfies GSRC with order $\tau$ and spectrum bounds: $c_{*}(\tau) = c_{-}(N)$ and $c^{*}(\tau) = c_{+}(N)$.
Consequently, the probability of satisfying SRC is smaller than that of satisfying GSRC, which illustrates the advantage of GSRC over SRC for group selection.
Our second condition is formally stated as:
\begin{itemize}
  \item [(C2)] ${\bm X}$ satisfies GSRC with order $2T$ and spectrum bounds $\{c_*(2T) , c^*(2T) \}$.
\end{itemize}
Condition (C2) requires that for any subset of groups $\A$ of size $|\A|\leqslant 2T$, the sub-matrix ${\bm X}_{\A}$ has full column rank, which is consistent with the assumption in GOMP \citep{ben2011near}.
If (C2) holds, the spectrum of the off-diagonal sub-matrices of ${\bm X}^\top {\bm X}/n$ can be bounded by a constant $\omega_T$. Specifically, $\omega_T$ is defined as the smallest number satisfying that:
\begin{equation*}
  \omega_T \|\bm u\|_2 \geqslant \dfrac{\|{\bm X}_{\A}^\top {\bm X}_{\mB} \bm u\|_2}{n},\ \forall  \bm u \neq 0, \bm u\in \mathbb{R}^{\#\{\mB\}}\ \text{with}\  |\A|\leqslant T, |\mB| \leqslant T, \text{and}\ \A \cap \mB = \emptyset.
\end{equation*}
Let $\delta_{T} = \max\{1-c_*(2T), c^*(2T)-1\}$. The constant $\delta_{T}$ is closely related to the block restricted isometry property (block RIP) constant for ${\bm X}$ \citep{eldar2009robust}.
In Lemma 4 of the online supplement, we show that $\omega_T$ can be bounded by $\delta_{T}$.

The third condition requires that:
\begin{itemize}
  \item [(C3)] $0\leqslant\mu_T<1$, where $\mu_T = \dfrac{8c^*(T)\left((1+\eta)\frac{\omega_T}{c_*(T)}(1+\frac{\omega_T}{c_*(T)})\right)^2}{(1-\eta)\left(c_*(T)-\frac{\omega_T^2}{c_*(T)}\right)}$ is a constant depending on $T$ for some constant $0 < \eta < 1$.
\end{itemize}

Condition (C3) restricts the correlation across the groups. In particular, when groups are uncorrelated, it can be verified that $\omega_T = 0$ and $c_*(2T)=c^*(2T)=1$.
For this ideal case, $\mu_T$ becomes zero.
As the correlation across the groups increases,
the spectrum bounds of GSRC expand away from 1, and $\omega_T$ increases. Consequently, $\mu_T$ increases away from zero.
To ensure (C3) holds, a sufficient condition is $\delta_{T} \leqslant 0.188$, i.e., $c_*(2T) \geqslant 0.812,\ c^*(2T)\leqslant 1.188$. 

For simplicity of notation, in what follows, we denote $\vartheta = \min\limits_{j\in \A^*}\|{\bm\beta}_{G_j}\|_2^2$ as the minimum group signal. Denote $\delta_1=O(p\exp\{-nC_1\vartheta/s^*p_{\max}\})$ and $\delta_2=O(p\exp\{-nC_2\vartheta/Tp_{\max}\})$  for some positive constants $C_1$ and $C_2$ depending on the spectrum bounds in GSRC. 
Finally, we require some necessary conditions for the magnitude of ${\bm\beta}^*$ and the threshold $\pi_T$:
\begin{itemize}
  \item [(C4)] The minimum group signal $\vartheta$ satisfies $\frac{Tp_{\max}\log p\log(\log n)}{n\vartheta} = o(1)$.
  \item [(C5)] The threshold in Algorithm \ref{alg:gsplicing} satisfies $\pi_T=O(\frac{Tp_{\max}\log p\log(\log n)}{n})$.
\end{itemize}
Condition (C4) requires a lower bound of the minimum group signal. 
It is a common and necessary assumption to achieve selection consistency of group selection \citep{W2010,H2012,GIGA}.
Compared with the standard splicing approach \citep{Zhu202014241}, a main advantage of GSplicing is that the minimum group signal, rather than the minimum individual signal, needs to be lower-bounded.
Condition (C5) assumes the threshold $\pi_T$ grows at an appropriate rate
since a large $\pi_T$ will miss the effective iterations. On the other hand, a small $\pi_T$ will increase the number of iterations as Corollary \ref{cor:step} shows.

\subsection{Statistical properties}\label{subsec:statistical}
In Section~\ref{subsec:support}, we first show that when $T \geqslant s^*$, the solution of GSplicing covers the true subset of groups $\A^*$ with high probability.
We refer to this property as the support recovery of GSplicing.
Next, in Section~\ref{subsec:selection}, we investigate the selection consistency of SGSplicing, that is, without any prior knowledge of $s^*$, SGSplicing is able to identify $\A^*$ correctly with high probability.

\subsubsection{Support recovery}\label{subsec:support}
\begin{theorem}\label{thm:no-false-exclusion}
Denote $(\hat \A, \hat \I, \hat {\bm\beta}, \hat {\bm d})$ as the solution of Algorithm 1. If (C1)-(C5) hold,
when $T \geqslant s^*$, we have
\begin{equation*}
P(\hat \A \supseteq \A^*) \geqslant 1 - \delta_1-\delta_2,
\end{equation*}
and, specifically, if $T=s^*$, we have
\begin{equation*}
  P(\hat \A = \A^*) \geqslant 1 - \delta_1-\delta_2.
  \end{equation*}
\end{theorem}

\noindent \textbf{Proof sketch}\textit{
Assume the output $\hat \A$ misses several relevant groups.
(C2) serves as a useful tool to bound the gap between the current loss and the loss after group splicing.
We can prove that, with probability at least $1-\delta_1-\delta_2$, $L(\hat {\bm\beta})$ decreases more than $\frac{(1-\mu_T)(1-\eta)(c_*(T)-\frac{\omega_T^2}{c_*(T)})}{2}\vartheta$ after group splicing.
Here we use Hoeffding's inequality to control the behavior of the sub-Gaussian distributed ${\bm \varepsilon}$, which defines the specific form of the probabilities $\delta_1$ and $\delta_2$.
Combining (C3)-(C5), the decrease of $L(\hat {\bm\beta})$ exceeds the threshold $\pi_T$, which contradicts the convergence criterion of Algorithm \ref{alg:gsplicing}.
\indent As a result, we can conclude that, with high probability, $\hat \A$ includes all groups in $\A^*$, 
and, specifically, $\hat \A = \A^*$ when $T =s^*$.
\qedsymbol}

Theorem~\ref{thm:no-false-exclusion} shows that
the output of Algorithm~\ref{alg:gsplicing} is a no-false-exclusion estimator with high probability.
In comparison with splicing \citep{Zhu202014241}, GSplicing recovers the true subset of groups with weaker assumptions on identifiability conditions and signals,
which shows that GSplicing is superior to the standard splicing for capturing group structure.
\begin{corollary}\label{coro:support}
Assume the conditions in Theorem~\ref{thm:no-false-exclusion} hold. If  $T \geqslant s^*$, we have
  $
    \lim\limits_{n\rightarrow \infty} P(\hat \A \supseteq \A^*) = 1,
  $
and, specifically, if $T=s^*$, we have
  $
    \lim\limits_{n\rightarrow \infty} P(\hat \A = \A^*) = 1.
  $
\end{corollary}
Corollary~\ref{coro:support} guarantees that
the solution of GSplicing includes all relevant groups with probability converging to one.
A similar support recovery property is held for the adaptive GLasso \citep{huang2010, W2010}.

\subsubsection{Selection consistency}\label{subsec:selection}
We require some reasonable assumptions on $p_{\max}$,  $\#\{\A^*\}$, and the maximum number of variables selected into the model $\#\{\hat\A_{T_{\max}}\}$.
We summarize them as conditions (C6) and (C7) below.
\begin{itemize}
  \item [(C6)] $\frac{\#\{\A^*\}\log J\log(\log n)}{n} = o(1)$ and $\frac{\#\{\hat\A_{T_{\max}}\}\log p}{n}=o(1)$.
  \item[(C7)] The maximum group size $p_{\max}$ satisfies $p_{\max} = o(J^{\log(\log n)})$.
\end{itemize}
Condition (C6) imposes a looser constraint on sparsity than the standard splicing.
Expressly, with a known group structure, GSplicing can guarantee a perfect recovery for a potentially higher sparsity level.
A similar result is obtained in GOMP \citep{BOMP}.
Condition (C7) assumes the logarithm of $p_{\max}$ grows at a slower rate than the penalty term in GIC, i.e., $\log J \log (\log n)$.
When $n$ and $J$ are large, a large $p_{\max}$ is allowed into the model.

\begin{theorem}\label{thm:consis}
  Denote $(\hat \A, \hat \I, \hat {\bm\beta}, \hat {\bm d})$ as the solution of Algorithm 2.  
  Assume (C2)-(C4) hold with $T_{\max}$ and (C1), (C5)-(C7) hold. For a sufficiently large $n$, we have
  \begin{equation*}
    P\left(\min\limits_{\hat \A \neq \A^*, \hat \A \subseteq \mS} \textup{GIC}(\hat\A) > \textup{GIC}(\A^*)\right) \geqslant 1 - O(p^{-\alpha}),
  \end{equation*}
  for some constant $0 < \alpha <1$.
\end{theorem}
\noindent \textbf{Proof sketch}\textit{
We separate the proof into two cases, i.e., $T < s^*$ and $T \geqslant s^*$.
The main idea is to bound the gap between the logarithmic losses output by Algorithm \ref{alg:gsplicing} given $T$ and $s^*$, and then compare its divergence rate with the penalty term.
\indent A useful sandwich relation to bound the gap is
\begin{equation}\label{brief2}
  \dfrac{L({\bm\beta}_1)-L({\bm\beta}_2)}{L({\bm\beta}_1)}\leqslant \log \dfrac{L({\bm\beta}_1)}{L({\bm\beta}_2)} \leqslant \dfrac{L({\bm\beta}_1)-L({\bm\beta}_2)}{L({\bm\beta}_2)},
\end{equation}
where ${\bm\beta}_1$ and ${\bm\beta}_2$ are the estimators of Algorithm \ref{alg:gsplicing} given model sizes $T_1$ and $T_2$, respectively. 
For $T = T_1 < T_2 = s^*$, we use the left-hand side of \eqref{brief2} to lower bound the gap and show that it diverges to infinity at a rate $O(n)$. 
In comparison, (C6) deduces that the penalty term $\#\{\A^*\}\log J \log(\log n)$ diverges at a rate $o(n)$, which implies $\text{GIC}(\hat\A_{T_1}) > \text{GIC}(\hat\A_{T_2})$ for a sufficiently large $n$.
For $T = T_1 \geqslant T_2 = s^*$, (C6) and (C7) establish the upper bound of $\A_{T_{\max}}$ and $p_{\max}$, which are necessary to guarantee that the right-hand side of \eqref{brief2} diverges at a slower rate than the penalty term.
As a result, we have $\text{GIC}(\hat\A_{T_1}) < \text{GIC}(\hat\A_{T_2})$.
\indent Therefore, we derive that GIC attains a minimum when $T=s^*$.
Finally, combining the conclusion of Theorem \ref{thm:no-false-exclusion}, we prove that Algorithm \ref{alg:sgsplicing} identifies $\A^*$ with high probability.
\qedsymbol}

Theorem~\ref{thm:consis} shows that, with high probability, SGSplicing can identify the true subset of groups $\A^*$ with an unknown model size.
Consequently, we can directly deduce that the estimator of SGSplicing is the same as the oracle least-squares estimator.
\begin{corollary}\label{cor:consis}
  Assume the conditions in Theorem \ref{thm:consis} hold. The estimator of Algorithm \ref{alg:sgsplicing} is an oracle estimator with high probability,
     $P(\hat {\bm\beta} = \hat {\bm\beta}^o) \geqslant 1 - O(p^{-\alpha}),$
  where $0 < \alpha < 1$ and $\hat {\bm\beta}^o$ is the least-squares estimator given the true subset of groups $\A^*$.
  \end{corollary}

\subsection{Convergence properties}\label{subsec:convergence_rate}
In this part, we first establish an upper bound of the convergence rate of loss. Next, we derive the maximum number of iterations when the selected set covers the true subset of groups $\A^*$.
Finally, we prove the polynomial complexity of SGSplicing with high probability.
\begin{theorem}\label{thm:conv_l2}
  Denote ($\A^k, \I^k, {\bm\beta}^k, \bm d^k$) as the results of Algorithm 1 in the $k$th iteration. Assume Conditions (C1)-(C5) hold. If $T \geqslant s^*$, then we have
  \begin{itemize}
    \item [(i)]
    \begin{equation*}
      |2nL({\bm\beta}^k)-2nL({\bm\beta}^*)| \leqslant \mu_T^k\| {\bm y}\|_2^2,
  \end{equation*}
    \item [(ii)]
    \begin{equation*}
      \A^k \supseteq \A^*,\ \text{if}\ k>\log_{\frac{1}{\mu_T}}\frac{\| {\bm y}\|_2^2}{(1-\frac{\eta}{2})n\left(c_*(T)-\frac{\omega_T^2}{c_*(T)}\right)\vartheta},
  \end{equation*}
  \end{itemize}
    with probability at least $1 - \delta_1-\delta_2$.
\end{theorem}
\noindent \textbf{Proof sketch}\textit{
The basic technique for the proof is similar to Theorem 1. Much effort is spent on deriving the inequality 
\begin{equation}\label{brief1}
  |2nL({\bm\beta}^{k+1}) - 2nL({\bm\beta}^*)|\leqslant \mu_T |2nL({\bm\beta}^{k}) - 2nL({\bm\beta}^*)|,\ k=0,1,\ldots.
\end{equation}
By repeated application of \eqref{brief1}, we obtain part (i).
The right-hand side of part (i) decays geometrically since (C3) requires $\mu_T < 1$. 
When $\A^{k+1}$ misses some relevant groups,  we can establish the lower bound of  $|2nL({\bm\beta}^{k+1}) - 2nL({\bm\beta}^*)|$ in terms of the spectrum bounds in (C2) and $\vartheta$.
Therefore, when $\mu_T^k\| {\bm y}\|_2^2$ is smaller than the lower bound, we can conclude that $\A^k \supseteq \A^*$.
\qedsymbol}
Part (\textrm{i}) in Theorem~\ref{thm:conv_l2}  shows that the estimation error of the loss can be bounded by a term related to $\| {\bm y}\|_2^2$.
These bounds decay geometrically until GSplicing converges, and the contraction factor $0<\mu_T<1$ establishes an upper bound of the convergence rate of loss,
which shows that loss converges at least linearly.
Part (\textrm{ii}) shows that the output of GSplicing covers the true subset of groups $\A^*$ after $O\left(\log_{\frac{1}{\mu_T}}\frac{\| {\bm y}\|_2^2}{(1-\frac{\eta}{2})n(c_*(T)-\frac{\omega_T^2}{c_*(T)})\vartheta}\right)$ iterations.
Moreover, the $\ell_2$ error bounds of the estimator ${\bm\beta}^k$ are immediate results of Theorem \ref{thm:conv_l2}, which are presented in Appendix B of the online supplement.
\begin{corollary}\label{cor:step}
  Assume the conditions in Theorem \ref{thm:conv_l2} hold. If $T \geqslant s^*$, Algorithm \ref{alg:gsplicing} stops after
  $O\left(\log_{\frac{1}{\mu_T}}\frac{\| {\bm y}\|_2^2}{(1-\frac{\eta}{2})n(c_*(T)-\frac{\omega_T^2}{c_*(T)})\vartheta}\right)$ iterations
  with probability at least $1 - O(p^{-\alpha})$ for some constant $0 < \alpha < 1$.
\end{corollary}
In Corollary \ref{cor:step}, we can show that once $\A^k$ covers the true subset of groups $\A^*$,
the decrease of loss exceeds the threshold $\pi_T$ in the next iteration. Thus, GSplicing converges after $O\left(\log_{\frac{1}{\mu_T}}\frac{\| {\bm y}\|_2^2}{(1-\frac{\eta}{2})n(c_*(T)-\frac{\omega_T^2}{c_*(T)})\vartheta}\right)$ iterations.
Combined with Corollary \ref{cor:step}, we prove the polynomial computational complexity of SGSplicing with high probability.
\begin{theorem}\label{thm:complexity}
  Assume (C2)-(C4) hold with $T_{\max}$ and (C1), (C5) hold. The computational complexity of Algorithm \ref{alg:sgsplicing} for a given $T_{\max}$ is
  \begin{equation*}
    O\left((T_{\max}\log_{\frac{1}{\mu_T}}\frac{\| {\bm y}\|_2^2}{p_{\max}\log p\log(\log n)}+\frac{n\| {\bm y}\|_2^2}{p_{\max}\log p\log(\log n)})(np+nT_{\max}p_{\max}+J)C_{\max}\right),
  \end{equation*}
  with probability at least $1-O(p^{-\alpha})$ for some constant $0 < \alpha < 1$.
\end{theorem}
\noindent \textbf{Proof sketch}\textit{
The proof is straightforward since the maximum number of iterations for $T<s^*$ and $T\geqslant s^*$ is accessible from the threshold $\pi_T$ and Corollary 3, respectively.
Next, we can analyze the computational complexity for each iteration. By multiplying these two parts, the total computational complexity of Algorithm $\ref{alg:sgsplicing}$ is obtained.
\qedsymbol}
In particular, the computational complexity provided in Theorem \ref{thm:complexity} is bounded by the sample size $n$, the number of groups $J$, the dimensionality $p$, the maximum model size $T_{\max}$, the maximum group size $p_{\max}$, and several logarithmic terms of them. 
This result indicates that SGSplicing terminates in polynomial time with high probability.

\vspace{0.5cm}

\section{\label{numerical}Numerical experiments}
This section is devoted to illustrating the empirical performance of our methods on both synthetic datasets (Section 4.1) and a real-world dataset (Section 4.2). 
We compare against several  
state-of-the-art methods: GLasso and GMCP, which are computed by R package $\mathtt{grpreg}$  \citep{B2015}, and GOMP, whose implementation follows {\citet{BOMP}} in R. 
We implement our methods in R package $\mathtt{abess}$ \citep{JMLR:v23:21-1060}.
For GLasso and GMCP, we determine the tuning parameter by BIC and 5-fold cross-validation (CV), respectively. Moreover, we leave the remaining hyperparameters to their default values in $\mathtt{\mathtt{grpreg}}$.
For GOMP and our methods, we select the model size by GIC.
All numerical experiments are conducted in R and executed on a personal laptop (Intel(R) Core(TM) i7-8550U, 1.80 GHz, 16.00 GB of RAM).
The code is available at \url{https://github.com/abess-team/A-Splicing-Approach-to-Best-Subset-of-Groups-Selection}.

\subsection{Synthetic dataset analysis}
Synthetic datasets are generated from a group linear model, 
where the elements of ${\bm \varepsilon}$ are generated independently with ${\bm \varepsilon}_i \sim \mathcal{N}(0,\sigma_1^2),\ i = 1,\ldots,n$.
We generate the design matrix ${\bm X}$ as follows. 
First, we generate an $n$-by-$J$ matrix $\widetilde{{\bm X}}$ of which each row is drawn independently from a multivariate Gaussian distribution $\mathcal{MVN}(\bm 0, \bm \Sigma)$. 
The covariance matrix $\Sigma$ is considered as one of the following two settings:
\begin{itemize}
  \item exponential correlation structure: ${\bm\Sigma}_{ij}=\rho^{|i-j|}$, 
  \item constant correlation structure: ${\bm\Sigma}_{ij} = \rho^{I(i \neq j)}$, 
\end{itemize}
where the constant $\rho  \in [0, 1]$ controls the correlation among the columns of $\widetilde{{\bm X}}$. 
A large $\rho$ implies that the columns of $\tilde{{\bm X}}$ are highly correlated.
Next, we consider a group size of $K$, and then generate $J \times K$ random vectors $\bm R_1, \ldots, \bm R_{JK} \in \mathbb{R}^n,$ whose entries are independently from $\mathcal{N}(0, 1)$.
Finally, we generate the design matrix ${\bm X}$ as
\begin{equation*}
  {\bm X}_{(j-1)K+k} = \frac{\tilde{{\bm X}}_j+\bm R_{(j-1)K+k}}{\sqrt{2}},\quad 1\leqslant j\leqslant J, 1\leqslant k\leqslant K.
\end{equation*}
Notably, the group structure generated by a large $\rho$, e.g., $\rho=0.9$, is highly correlated across the groups.
The underlying regression coefficients ${\bm\beta}^*$ are generated in the following way.
For a relevant group $j \in \A^*$, 
the $i$th element of ${\bm\beta}^*_{G_j}$ is set as
\begin{equation*}
  ({\bm\beta}^*_{{G_j}})_i = \gamma_i^j - \dfrac{1}{K+1}\sum\limits_{i=1}^{K+1} \gamma^j_i, \ 1\leqslant i \leqslant K,
\end{equation*}
where $\gamma_1^j, \ldots, \gamma_{K+1}^j$ are independently drawn from $\mathcal{N}(0, 1)$.

Given an output $(\hat \A, \hat {\bm\beta})$,  
we use the following metrics to assess the group selection, model size selection, and parameter estimation:
\begin{itemize}
  \item \textbf{True Positive Rate (TPR)}: TP/(TP+FN), 
  where TP is the size of the intersection between $\hat \A$ and $\A^*$, FN is the size of the intersection between $\hat \I$ and $\A^*$.
  \item \textbf{False Positive Rate (FPR)}: FP/(FP+TN),
  where TN is the size of the intersection between $\hat \I$ and $\I^*$, FP is the size of the intersection between $\hat \A$ and $\I^*$.
  \item \textbf{Mathews Correlation Coefficient (MCC)}:
  \begin{align*}
    \text{MCC} = \frac{\text{TP}\times \text{TN}-\text{FP}\times \text{FN}}{\sqrt{(\text{TP+FP)(TP+FN)(TN+FP)(TN+FN)}}}.
  \end{align*}
  MCC ranges in $[-1, 1]$, and a larger MCC means better performance on group selection.
  \item \textbf{Group Sparsity Error (GSE)}: $|\hat \A| - |\A^*|$.
  \item \textbf{Relative Estimation Error (ReEE)}: $\text{ReEE} = \|\hat {\bm\beta} - {\bm\beta}^*\|_2/\|{\bm\beta}^*\|_2$.
\end{itemize}
Also, we provide the runtime in seconds for each method. 
Due to the high computational burden of CV, we remove the results of GLasso and GMCP tuned by CV at runtime.
All simulation results are based on 100 replications.
In Section 4.1.2-4.1.5, we use boxplots to help visualize the distribution of the computational results of each method, where the black dots represent outliers. 
Notably, the flat box in the boxplot implies that the method has stable performance on this metric.

\subsubsection{Discussion of $C_{\max}$}\label{cmax}
Here we present the empirical evidence of the choice of $C_{\max}$. We consider that $\tilde{{\bm X}}$ has an exponential correlation structure with $\rho = 0.9$.
We set the number of groups $J=2000$, the group size $K=5$ and the sample size $n = 1000$.
Additionally, we set $\sigma_1=3$. The true subset of groups $\A^*$ is randomly chosen with $s^*=10$.
We set different values of $C_{\max}=1,2,5,10$.
\begin{table}[H]
  \centering{
  \caption{\label{tab1} Comparison of SGSplicing and GGSplicing with different values of $C_{\max}$. The standard deviations are shown in the parentheses.}
  \begin{tabular}{*{8}c}
    \toprule
    $C_{\max}$& Method & TPR & FPR & MCC & ReEE & Runtime \\
    \midrule
    \multirow{2}*{1}
    &SGSplicing & 0.99  (0.04) & 0.00  (0.00) & 0.99  (0.02) & 0.15  (0.03) & 1.85 (0.36) \\
    &GGSplicing & 0.98  (0.05) & 0.00  (0.00) & 0.95  (0.04) & 0.16  (0.03) & 0.85 (0.18)\\
    \midrule
    \multirow{2}*{2}
    &SGSplicing & 0.99  (0.04) & 0.00  (0.00) & 0.99  (0.02) & 0.15  (0.03) & 1.47 (0.31) \\
    &GGSplicing & 0.99  (0.05) & 0.00  (0.00) & 0.97 (0.04) & 0.16  (0.03) & 0.59 (0.15)\\
    \midrule
    \multirow{2}*{5}
    &SGSplicing & 0.99  (0.04) & 0.00  (0.00) & 0.99  (0.02) & 0.15 (0.03) & 1.87 (0.37) \\
    &GGSplicing & 0.99  (0.05) & 0.00  (0.00) & 0.97 (0.04) & 0.16 (0.03) & 0.76 (0.20)\\
    \midrule
    \multirow{2}*{10}
    &SGSplicing & 0.99  (0.04) & 0.00 (0.00) & 0.99 (0.02) & 0.15  (0.03) & 2.10 (0.38) \\
    &GGSplicing & 0.98  (0.05) & 0.00 (0.00) & 0.97 (0.04) & 0.16 (0.03) & 0.93 (0.21)\\
    \bottomrule
  \end{tabular}}
\end{table}

Table \ref{tab1} summarizes the results, from which we see that the results are nearly identical, except the runtime among different settings of $C_{\max}$. 
Obviously, when $C_{\max}=2$, the runtime attains the minimum value, and the performance achieves the best.
In summary, a small $C_{\max}$, e.g., $C_{\max}=1$, increases the number of iterations. On the other hand, a large $C_{\max}$, e.g., $C_{\max}=5$ or $10$, brings several redundant updates in each iteration, although it will decrease the number of iterations.
Consequently, we set $C_{\max}=2$ in Algorithm \ref{alg:gsplicing} empirically.

\subsubsection{Influence of the correlation across the groups}
We consider the following settings for $\tilde{{\bm X}}$:
(i) exponential correlation structure with $\rho = 0.6$ or $\rho = 0.9$;
(ii) constant correlation structure with $\rho = 0.6$ or $\rho = 0.9$.
We set the number of groups $J=1500$, the group size $K=3$, and the sample size $n=500$.
Additionally, we set $\sigma_1=2$, and consider $\A^*$ is randomly chosen with $s^* = 15$.
\begin{figure}[h]
  \centering
  \includegraphics[width=\linewidth]{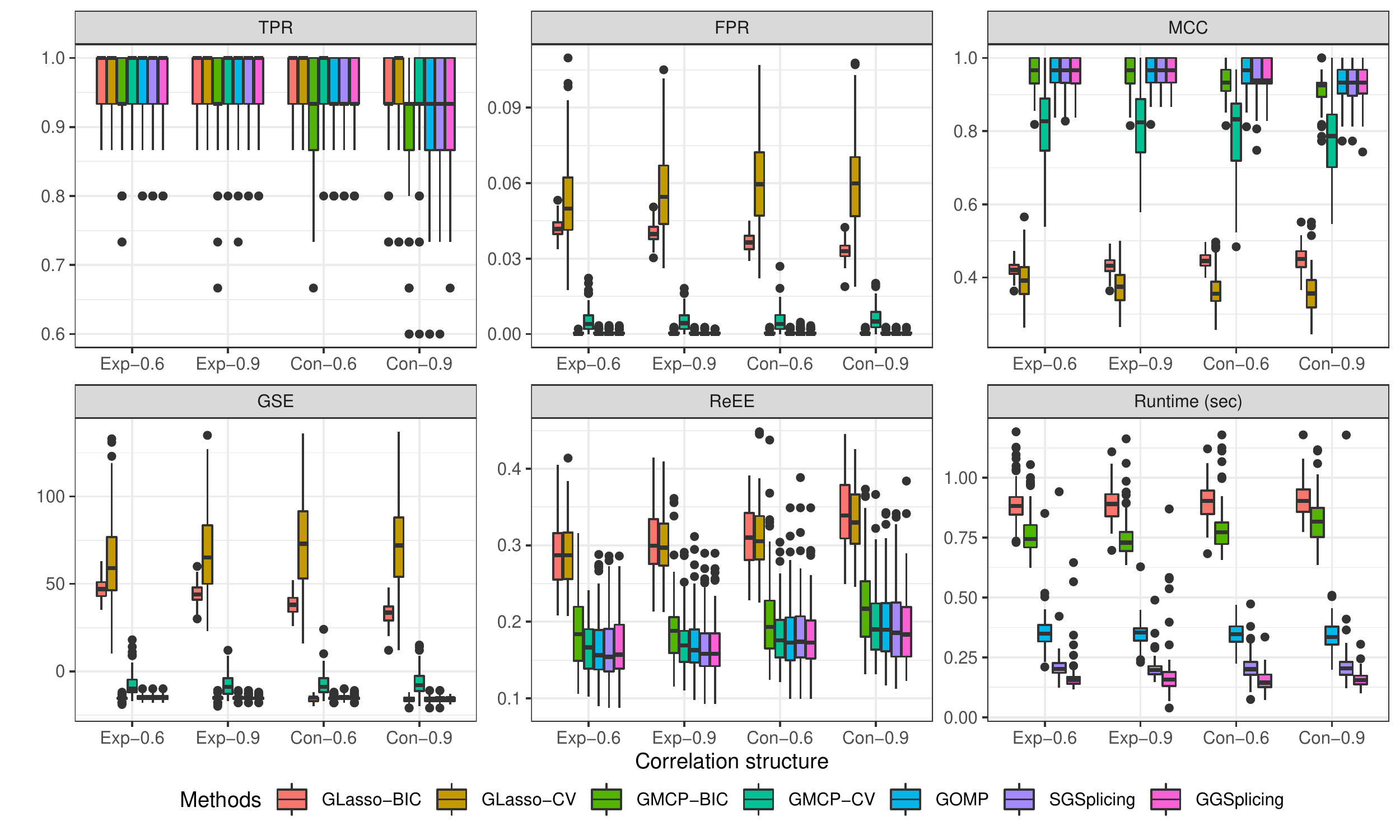}
  \caption{\label{simu1}The boxplots of all metrics of different methods with different correlation structures.
  Exp-0.6 corresponds to exponential correlation structure with $\rho = 0.6$, 
  and Con-0.6 corresponds to constant correlation structure with $\rho = 0.6$.  
  Con-0.9 and Exp-0.9 have similar meanings.}
\end{figure}

From Figure~\ref{simu1}, it is evident that 
a higher correlation across different groups decreases TPR and MCC but increases FPR and ReEE for each method. 
Consequently, we can conclude that as the correlation increases, group selection becomes more difficult.
From the GSE presented in Figure~\ref{simu1}, we see that GLasso-BIC and GLasso-CV select far more groups.
As a result, these methods have relatively poor performance on FPR, which is consistent with the evidence that GLasso tends to overestimate the model size. 
In comparison, other methods control the level of FPR at a relatively lower level, which leads to better performance on MCC.
Additionally, SGSplicing and GGSplicing achieve competitive or even the best performance on MCC and GSE. 
Furthermore, the ReEE of our methods outperforms GLasso and GMCP because of the unbiased estimate given by our methods. 
These results indicate that our methods outperform the other state-of-the-art methods in terms of group selection and parameter estimation when the group structure is highly correlated across the groups.
Finally, in terms of runtime, our methods are much faster than GLasso-BIC and GMCP-BIC. 
It is worthy to note that GGSplicing not only accelerates the selection procedures efficiently
but also preserves a competitive performance on group selection.

\subsubsection{Influence of the sample size}
We consider that $\tilde{{\bm X}}$ has a constant correlation structure with $\rho = 0.6$. 
We set $J=1000$, $K=5$, and the standard deviation of noise $\sigma_1=3$. Let $\A^*$ be randomly chosen with $s^* = 10$. 
The sample size $n$ is set at $500, 600$ and $700$. 
\begin{figure}[h]
  \centering
  \includegraphics[width=\linewidth]{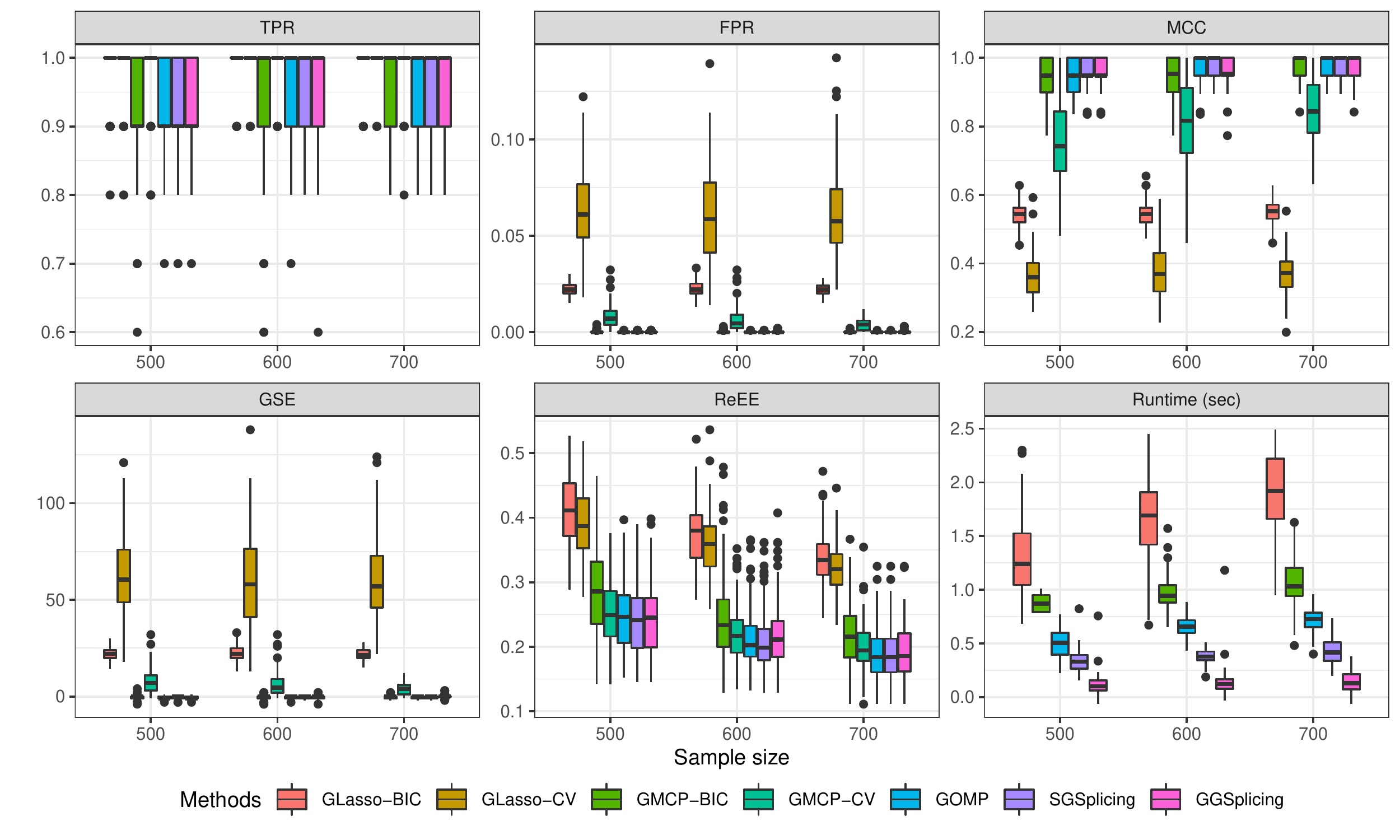}
  \caption{\label{simu2}The boxplots of all metrics of different methods when sample size increases.}
\end{figure}

The simulation results are shown in Figure~\ref{simu2}, from which we see that the performance of all methods becomes better as $n$ increases. 
Notably, our methods perform better than the others, especially when $n = 500$.
Meanwhile, the improvements of our methods are significant as $n$ increases,
and the runtime of our methods is significantly smaller than the other three methods.

\subsubsection{Influence of the number of groups}
We study the empirical performance of group selection methods when $J$ increases to 500, 1000, or 1500.
We consider an exponential correlation structure with $\rho = 0.9$ for $\tilde{{\bm X}}$. 
We set $n=500$ and $K=4$.
The remaining settings are consistent with Section 4.1.3.

\begin{figure}[h]
  \begin{center}
    \includegraphics[width=\linewidth]{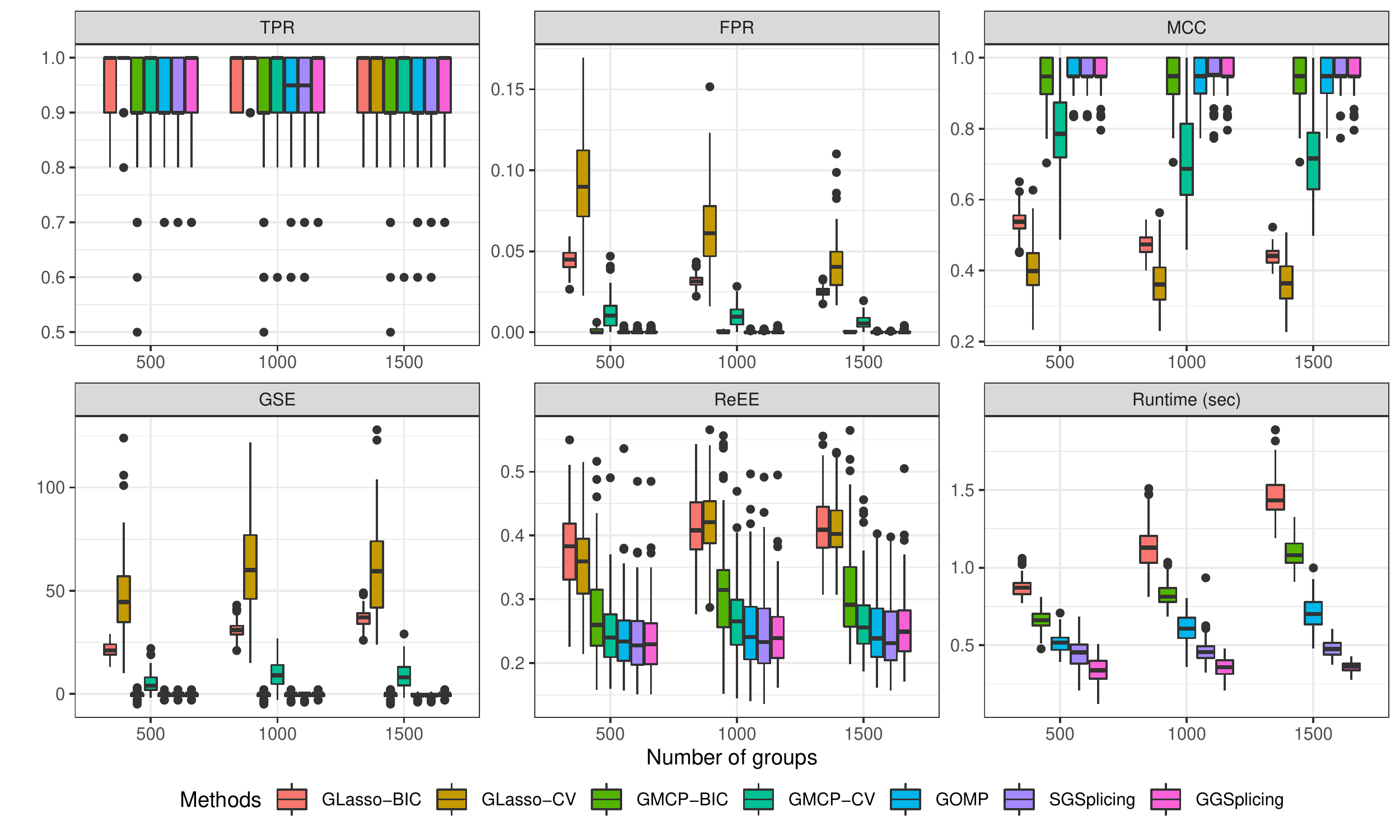}
    \caption{The boxplots of all metrics of different methods when the number of groups increases.}\label{simu3}
  \end{center}
\end{figure}
From Figure \ref{simu3}, we see that all methods perform worse as $J$ increases.
However, SGSplicing and GGSplicing outperform the others, especially on MCC and ReEE, indicating that our methods are more robust to high dimensionality.
Moreover, the runtime of GLasso, GMCP, and GOMP increases significantly as $J$ increases. 
In comparison, our methods significantly shorten the runtime.

\subsubsection{Influence of the group size}\label{simulation}
We consider that $\tilde{{\bm X}}$ has a constant correlation structure with $\rho = 0.9$. We set $\sigma_1=4$, $n=1000$, $J = 1000$ and $K=5, 10, 15$, respectively. 
The true subset of groups $\A^*$ is randomly chosen with $s^* = 5$. 

\begin{figure}[htbp]
  \centering
  \includegraphics[width=\linewidth]{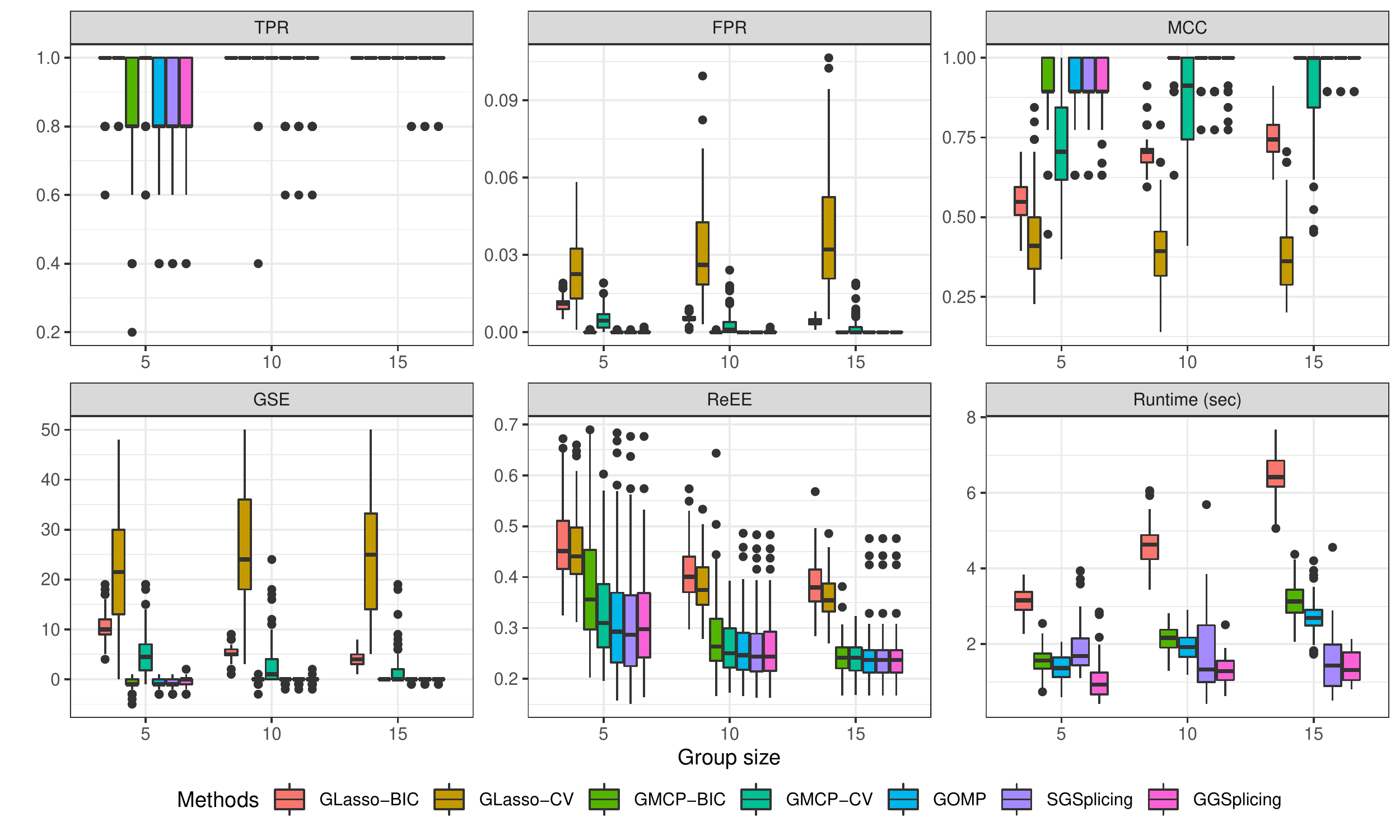}
  \caption{\label{simu4}The boxplots of all metrics of different algorithms when group size increases.}
\end{figure}
The simulation results are shown in Figure~\ref{simu4}. 
When the group size is small, e.g., $K=5$, the correlation across the groups is relatively high.
As $K$ increases, 
the magnitude of the group signal enhances, and the effect of correlation becomes smaller, 
which leads to better performance for all methods. 
Figure~\ref{simu4} shows that our methods enjoy competitive performance on all metrics.

\subsubsection{Computational complexity analysis}\label{computational_time}
In this part, we provide empirical evidence to support Theorem~\ref{thm:complexity}. 
We are interested in the effects of the number of groups $J$, the sample size $n$, 
and the maximum group size $p_{\max}$ on the computational complexity of SGSplicing and GGSplicing. 
Note that $T_{\max}$ for both SGSplicing and GGSplicing is $[\frac{n}{p_{\min}\log p}]$, which is related to the three components we consider here. 
To study the effects of these components, we 
designate the following three settings:
\begin{itemize}
  \item [\textbf{A}:] $n=1000$ and $p_{\max} = 3$. $J$ increases from 700 to 1000 with an increment of 30.
  \item [\textbf{B}:] $J=1000$ and $p_{\max} = 3$. $n$ increases from 1000 to 1500 with an increment of 100.
  \item [\textbf{C}:] $n=500$ and $J = 500$. $p_{\max}$ increases from 3 to 10 with an increment of 1.
\end{itemize}
Additionally, we set $\sigma_1=2$ and $s^*=10$, and consider an exponential correlation structure with $\rho=0.6$.
Note that $p_{\max} = p_{\min} = K$ here.

\begin{figure}[h]
  \centering
  \includegraphics[scale = 0.5]{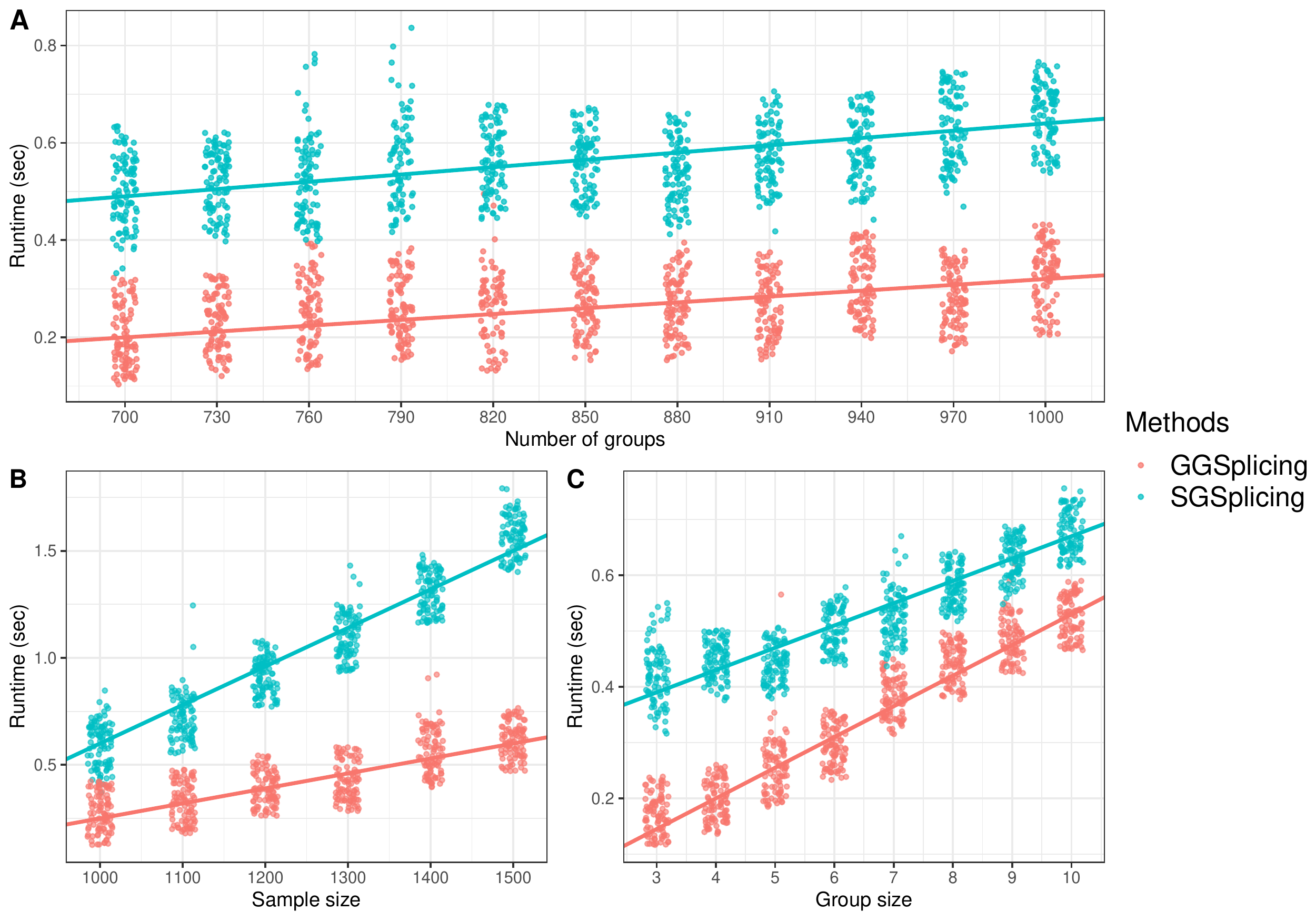}
  \caption{\label{computational_t}(A) Number of groups ($x$-axis) versus Runtime ($y$-axis) scatterplot. 
  (B) Sample size ($x$-axis) versus Runtime ($y$-axis) scatterplot. 
  (C) Group size ($x$-axis) versus Runtime ($y$-axis) scatterplot.
  The green and red straight lines are characterized by equation $y = a + bx$, in which slopes and intercepts 
  are estimated by the ordinary least squares.}
\end{figure}

From Figure~\ref{computational_t}, it becomes clear that 
when any two components are fixed, 
all the curves grow in a near-linear fashion as the remaining component increases.
These results coincide with the conclusion of Theorem \ref{thm:complexity}. 
Several observations about the slopes of these two straight lines are worthy of being noted. 
Figure~\ref{computational_t}A reveals that as $J$ increases, $T_{\max}$ remains relatively the same due to the slow growth of the logarithmic term, 
leading to two nearly parallel straight lines.
In Figure~\ref{computational_t}B, it is apparent that $T_{\max}$ increases with the sample size $n$ significantly. Meanwhile,
the straight line of GGSplicing has a smaller slope than SGSplicing, which reveals that the acceleration of the golden-section strategy is efficient.
In contrast, $p_{\max}$, as one part of the denominator, draws an opposite conclusion.

\subsection{Real-world dataset analysis}
We consider a genomic dataset in a rat eye disease study \citep{s2006}.
The dataset consists of 120 twelve-week-old male rats, which collects the expression of TRIM32,
a gene that has been shown to cause Bardet-Biedl syndrome \citep{chiang2006homozygosity}, and other 18,975 related genes that potentially influence the expression of TRIM32.
Although there are numerous potential genes, we expect only a small subset of these genes related to the expression of TRIM32 \citep{fan2011, huang2010,B2015}.

Following \citet{fan2011}, we focus our interest on a subset with 2,000 probe sets. In particular, we select 2,000 probe sets that have the highest marginal ball correlation \citep{pan2019} with TRIM32 and conduct this feature screening procedure by R package $\mathtt{Ball}$ \citep{JSSv097i06}.
Next, we consider a five-term natural cubic spline basis expansion of these genes,
resulting in a high-dimensional group selection problem with sample size $n=120$ and dimensionality $p=10,000$.
For SGSplicing and GGSplicing, due to the relatively small default value of $T_{\max}$, we set $T_{\max}=10$.
The 120 rats are randomly split into a
training set with 100 samples and a test set with the remaining 20 samples.
We replicate these randomly-splitting procedures 100 times and 
compute the average of the numbers of groups selected and the prediction mean square error (PE) in the test set.

\begin{table}[htbp]
  \caption{\label{tab2}Computational results for TRIM32 dataset. The standard deviations are shown in the parentheses.}
  \centering
  \setlength{\tabcolsep}{6mm}{
  \begin{tabular}{*{3}c}
    \toprule
    Method & Number of groups & 100 $\times$ PE   \\
    \midrule
    GLasso-BIC& 32.01  (3.83) & 1.03  (1.38) \\
    GLasso-CV& 33.35  (15.70) & 1.19  (1.99) \\
    GMCP-BIC & 1.21  (0.77)& 1.14  (0.85) \\
    GMCP-CV & 4.30  (3.49)& 1.18  (0.99) \\
    GOMP & 1.01 (0.10) & 1.16 (2.29)\\
    SGSplicing  & 1.03  (0.17) & 0.94 (0.54) \\
    GGSplicing & 1.03 (0.17) & 0.94 (0.54) \\
    \bottomrule
  \end{tabular}}
\end{table}

From Table~\ref{tab2}, we see that both GLasso-BIC and GLasso-CV select far more groups than the other methods, but this does not lead to the best prediction performance on the test set. 
In comparison, GMCP-BIC and GMCP-CV select fewer groups than GLasso.
However, these sparser models bring relatively high PE to GMCP.
For GOMP and our methods, all methods tend to select approximately one group into the model, but the prediction error of GOMP is much higher than our methods.
Overall, these results demonstrate the superiority of our methods on predictive accuracy.

Furthermore, the whole $120$ samples are used to learn a sparse group linear model for the expression of TRIM32.
Both SGSplicing and GGSplicing select a model with only one group: 1373534\_at. 
Notably, the gene at probe set 1373534\_at is also considered as an important gene related to the expression of TRIM32 by \citet{fan2011} and \citet{zhou2018model}.
\begin{figure}[H]
  \centering
  \includegraphics[scale = 0.85]{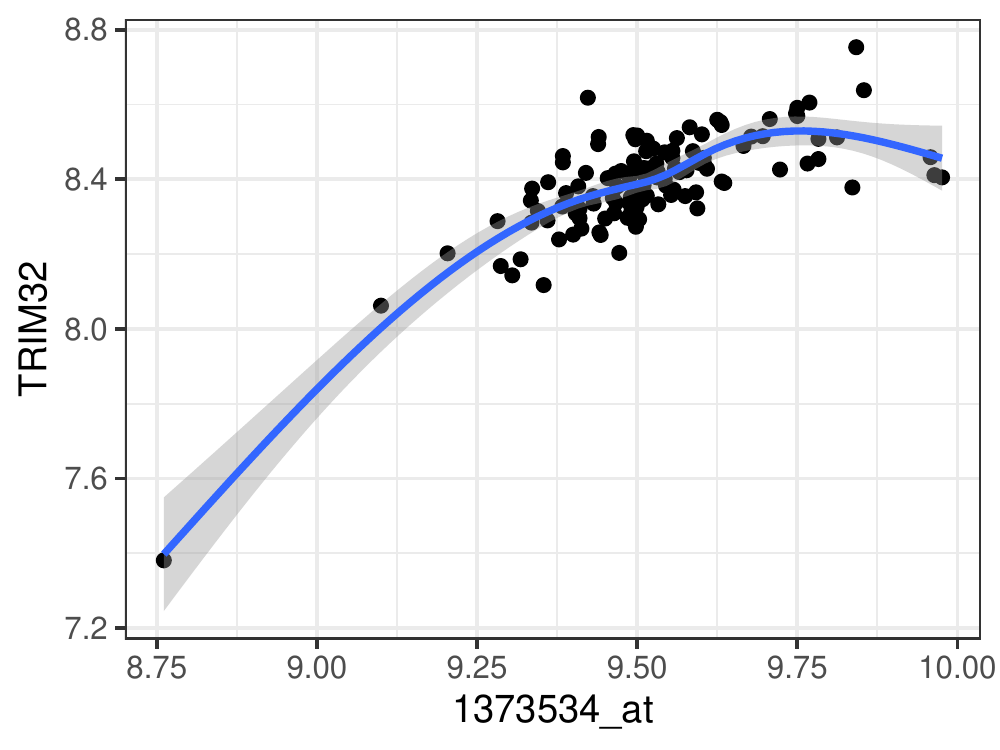}
  \caption{The scatterplot between group 1373534\_at and TRIM32, where the blue line is the linear model fitted by group 1373534\_at and TRIM32. The gray band is the confidence interval of the linear model.}\label{real1}
\end{figure}
In Figure~\ref{real1}, the group linear model based on group 1373534\_at can measure the non-linear relationship between 1373534\_at and TRIM32 efficiently. 
In fact, this one-group linear model can explain $71.2\%$ of the variance in the expression of TRIM32.
Additionally, we calculate the group selection frequencies via stability selection \citep{meinshausen2010stability} based on 100 replications 
and report the five most frequently selected groups in Table~\ref{tab3}. 
SGSplicing identifies group 1373534\_at with the highest selection frequency, $62\%$, 
which takes values in the range from $60\%$ to $90\%$ recommended by \citet{meinshausen2010stability}.
By contrast, the other groups are not very likely to be selected.

\begin{table}[ht]
  \centering
  \caption{The 5 groups with the highest selection frequencies selected by SGSplicing based on 100 replications.}\label{tab3}
  \begin{tabular}{cccccc}
    \toprule
    & 1373534\_at  &       1374669\_at &           1368136\_at    &       1376747\_at    &       1376180\_at    \\
    \midrule      
    Frequency & 62\% & 27\% & 23\% & 23\% & 21\% \\
    \bottomrule
  \end{tabular}
\end{table}

\section{Conclusion}\label{sec:conclusion}
GSplicing is an efficient group selection method for non-overlapping group structure that can select the subset of groups with an exact model size. 
Specifically, we derive the optimal conditions of the augmented Lagrangian form of BSGS and approximate these optimal conditions by a group-splicing approach. 
To better adapt to high dimensionality, we propose a novel information criterion called GIC.
Coupled with GIC, we develop an adaptive algorithm, SGSplicing, to determine the optimal model size. 
Furthermore, we demonstrate that SGSplicing is capable of perfectly recovering the true subset of groups in polynomial time with high probability. 
Motivated by the unimodal-like solution path of SGSplicing, we apply a golden-section technique to accelerate the selection procedures.  
Additionally, we conduct a complete theoretical analysis under certain assumptions, including the statistical and convergence properties.
Finally, the numerical experiments illustrate that our methods have more accurate and robust statistical performance than other state-of-the-art methods.

Recently, group selection for overlapping group structure has become a popular and practical topic for research \citep{jacob2009group, jenatton2011structured, jain2016structured, won2020group}. 
This topic is not explored in this paper, but it would be constructive to extend our method from non-overlapping to overlapping group structure in our future work.

\appendix



\section{Technical proofs}\label{app-sec:proof}
\subsection{Proof of Lemma 1}
\begin{proof}
Denote $\A = \{j\in \mS : I(\|{\bm{\upsilon}}_{G_j}\|_2 \neq 0)\}$ and $\I = \A^c$.
Denote
\vspace{-0.cm}
\begin{align*}
  f_{\kappa}({\bm{\beta}}, {\bm{\upsilon}}, {\bm{d}}) =& \frac{1}{2n}\|{\bm{y}}-{\bm{X}}{\bm{\beta}}\|_2^2+{\bm{d}}^\top({\bm{\beta}}-{\bm{\upsilon}})+\frac{\kappa}{2}\|{\bm{\beta}}-{\bm{\upsilon}}\|_2^2.
\end{align*}

\vspace{-0.1cm}

\noindent Given $({\bm{\upsilon}}, {\bm{d}})$, $f_{\kappa}({\bm{\beta}}, {\bm{\upsilon}}, , {\bm{d}})$ is minimized with respect to ${\bm{\beta}}$ that satisfies
\vspace{-0.2cm}
\begin{equation}\label{lem15}
  -{\bm{X}}^\top({\bm{y}}-{\bm{X}}{\bm{\beta}})/n+{\bm{d}}+\kappa ({\bm{\beta}} - {\bm{\upsilon}})=0.
\end{equation}

\vspace{-0.1cm}

\noindent Given $({\bm{\beta}}, {\bm{d}})$, $f_{\kappa}({\bm{\beta}}, {\bm{\upsilon}}, {\bm{d}})$ can be minimized with respect to ${\bm{\upsilon}}$. 
For this case, we have
\begin{align*}
  f_{\kappa}({\bm{\beta}}, {\bm{\upsilon}}, {\bm{d}}) &\propto {\bm{d}}^\top({\bm{\beta}}-{\bm{\upsilon}})+\dfrac{\kappa}{2}\|{\bm{\beta}}-{\bm{\upsilon}}\|_2^2 = \dfrac{\kappa}{2}\|({\bm{\beta}}+\dfrac{1}{\kappa}{\bm{d}})-v\|_2^2.
\end{align*}
By the constraint $\|{\bm{\upsilon}}\|_{0, 2} =T$,
$f_{\kappa}({\bm{\beta}}, {\bm{\upsilon}}, {\bm{d}})$ is minimized by choosing $\A$ as the $T$ largest $\|{\bm{\beta}}_{G_j}+\dfrac{1}{\kappa}{\bm{d}}_{G_j}\|_2^2$, i.e., 
\vspace{-0.2cm}
\begin{align}\label{lem13}
  \begin{split}
  \A=&\{j\in \mS:\sum_{i=1}^{J}I(\|{\bm{\beta}}_{G_j}+\frac{1}{\kappa}{\bm{d}}_{G_j}\|_2^2\leqslant \|{\bm{\beta}}_{G_i}+\frac{1}{\kappa}{\bm{d}}_{G_i}\|_2^2)\leqslant T\}.
  \end{split}
\end{align}

\vspace{-0.1cm}

\noindent Therefore, we have 
\begin{align}\label{lem14}
  {\bm{\upsilon}}_{G_j} = [\Lambda ({\bm{\beta}}, {\bm{d}})]_j =
  \begin{cases}
    {\bm{\beta}}_{G_j} + \frac{1}{\kappa}{\bm{d}}_{G_j},\ &\text{if}\ j \in \A\\
    0, &\text{if}\ j \in \I
  \end{cases}.
\end{align}
Given $({\bm{\beta}}, {\bm{\upsilon}})$, $f_{\kappa}({\bm{\beta}}, {\bm{\upsilon}}, {\bm{d}})$ is minimized with respect to ${\bm{d}}$ that satisfies
\vspace{-0.1cm}
\begin{equation}\label{lem16}
  {\bm{\beta}} - {\bm{\upsilon}}=0.
\end{equation}
From \eqref{lem15}-\eqref{lem16}, 
we derive that a coordinate-wise minimizer $({\bm{\beta}}^\diamond, {\bm{\upsilon}}^\diamond, {\bm{d}}^\diamond)$ of $f_{\kappa}({\bm{\beta}}, {\bm{\upsilon}}, {\bm{d}})$ satisfy
\begin{align*}
  \begin{cases}
  &{\bm{\upsilon}}^\diamond = \Lambda ({\bm{\beta}}^\diamond, {\bm{d}}^\diamond),\\
  &-{\bm{X}}^\top({\bm{y}}-{\bm{X}}{\bm{\beta}}^\diamond)/n+{\bm{d}}^\diamond+\kappa ({\bm{\beta}}^\diamond - {\bm{\upsilon}}^\diamond)=0,\\
  &{\bm{\beta}}^\diamond - {\bm{\upsilon}}^\diamond=0.
\end{cases}
\end{align*}
By simple algebra, we derive the necessary optimal conditions as
\vspace{-0.1cm}
\begin{align*}
  &{\bm{\beta}}^\diamond_{\A^\diamond} = ({\bm{X}}^\top_{\A^\diamond} {\bm{X}}_{\A^\diamond})^{-1} {\bm{X}}_{\A^\diamond}^\top {\bm{y}},\ {\bm{\beta}}^\diamond_{\I^\diamond} = 0,\\
  &{\bm{d}}^\diamond_{\A^\diamond}=0,\ {\bm{d}}^\diamond_{\I^\diamond}={\bm{X}}^\top_{\I^\diamond}({\bm{y}}-{\bm{X}}{\bm{\beta}}^\diamond)/n,\\
  &{\bm{\upsilon}}^\diamond = {\bm{\beta}}^\diamond,\\
  &\A^\diamond =\{j\in \mS:\sum\limits_{i=1}^{J}I(\|{\bm{\beta}}_{G_j}^\diamond+\frac{1}{\kappa}{\bm{d}}^\diamond_{G_j}\|_2^2\leqslant \|{\bm{\beta}}_{G_i}^\diamond+\frac{1}{\kappa}{\bm{d}}^\diamond_{G_i}\|_2^2)\leqslant T\},\\
  &\I^\diamond = (\A^\diamond)^c.
\end{align*}
\qedsymbol
\end{proof}

\subsection{Proof of Lemma 2}
\begin{proof}
For part $(i)$, note that
\begin{align*}
  \begin{split}
    L({\bm{\beta}}^{\A^k\backslash j})-L({\bm{\beta}}^k)
    &=\frac{1}{2n}\|{\bm{y}}-{\bm{X}}{\bm{\beta}}^{\A^k\backslash j}\|_2^2-\frac{1}{2n}\|{\bm{y}}-{\bm{X}}{\bm{\beta}}^k\|_2^2\\
    &=\frac{1}{2n}\left(({\bm{\beta}}^{\A^k\backslash j})^\top {\bm{X}}^\top {\bm{X}} {\bm{\beta}}^{\A^k\backslash j}-2{\bm{y}}^\top {\bm{X}} {\bm{\beta}}^{\A^k\backslash j}\right)-\frac{1}{2n} \left( ({\bm{\beta}}^k)^\top {\bm{X}}^\top {\bm{X}} {\bm{\beta}}^k - 2{\bm{y}}^\top {\bm{X}} {\bm{\beta}}^k \right)\\
    &=({\bm{\beta}}^{\A^k\backslash j} - {\bm{\beta}}^k)^\top \frac{{\bm{X}}^\top {\bm{X}}}{2n} ({\bm{\beta}}^{\A^k\backslash j} - {\bm{\beta}}^k)+\frac{1}{n}({\bm{\beta}}^{\A^k\backslash j} - {\bm{\beta}}^k)^\top {\bm{X}}^\top ({\bm{y}}-{\bm{X}}{\bm{\beta}}^k)\\
    &=({\bm{\beta}}^k_{G_j})^\top \frac{{\bm{X}}_{G_j}^\top {\bm{X}}_{G_j}}{2n} {\bm{\beta}}^k_{G_j}+\frac{1}{n}{\bm{\beta}}_{G_j}^k {\bm{X}}_{G_j}^\top ({\bm{y}}-{\bm{X}}{\bm{\beta}}^k)\\
    &=\frac{1}{2}\|{\bm{\beta}}^k_{G_j}\|_2^2,
  \end{split}
\end{align*}
where the last equality follows from $\dfrac{{\bm{X}}_{G_j}^\top {\bm{X}}_{G_j}}{n} = \bm{\mathrm{I}}_{p_j}$ and ${\bm{d}}_{G_j} = {\bm{X}}_{G_j}^\top({\bm{y}}-{\bm{X}}{\bm{\beta}})/n=0,\ j \in \A$.

Next, we prove part $(ii)$. From the maximum profile likelihood estimator, we have the $j$th group of ${\bm{t}}^k_j$ are $({\bm{X}}_{G_j}^\top {\bm{X}}_{G_j})^{-1}{\bm{X}}_{G_j}^\top({\bm{y}}-{\bm{X}}{\bm{\beta}}^k)=(\dfrac{{\bm{X}}_{G_j}^\top {\bm{X}}_{G_j}}{n})^{-1}{\bm{d}}_{G_j}^k$. 
Similarly, 
we have
\begin{align*}
  \begin{split}
  L({\bm{\beta}}^k)-L({\bm{\beta}}^k+{\bm{t}}^k_j)&=\frac{1}{2n}\|{\bm{y}}-{\bm{X}}{\bm{\beta}}^k\|_2^2-\frac{1}{2n}\|{\bm{y}}-{\bm{X}}({\bm{\beta}}^k+{\bm{t}}^k_j)\|_2^2\\
  &=-({\bm{t}}^k_j)^\top \frac{{\bm{X}}^\top {\bm{X}}}{2n} {\bm{t}}^k_j+\frac{1}{n}{\bm{t}}_j^\top {\bm{X}}^\top({\bm{y}}-{\bm{X}}{\bm{\beta}}^k)\\
  &=-\frac{1}{2}({\bm{d}}_{G_j}^k)^\top (\frac{{\bm{X}}_{G_j}^\top {\bm{X}}_{G_j}}{n})^{-1} {\bm{d}}^k_{G_j}+({\bm{d}}^k_{G_j})^\top (\frac{{\bm{X}}_{G_j}^\top {\bm{X}}_{G_j}}{n})^{-1} {\bm{X}}_{G_j}^\top({\bm{y}}-{\bm{X}}{\bm{\beta}}^k)/n\\
  &=\frac{1}{2}({\bm{d}}_{G_j}^k)^\top (\frac{{\bm{X}}_{G_j}^\top {\bm{X}}_{G_j}}{n})^{-1} {\bm{d}}^k_{G_j}\\
  &=\frac{1}{2}\|{\bm{d}}_{G_j}^k\|_2^2,
    \end{split}
\end{align*}
where the third equality uses the definition of ${\bm{t}}_j^k$.
\qedsymbol
\end{proof}

\subsection{Proof of Lemma 3}
\begin{proof} 
Given $C < |\A^k|$,
we have
\begin{align*}
  \min_{j \in \mS^k_{C, 2}} \frac{1}{\kappa^2}\|{\bm{d}}^k_{G_j}\|_2^2 = \min_{j \in \mS^k_{C, 2}} \|{\bm{\beta}}^k_{G_j}+\frac{1}{\kappa}{\bm{d}}^k_{G_j}\|_2^2 \geqslant \max_{i \in \A^k}\|{\bm{\beta}}^k_{G_i}+\frac{1}{\kappa}{\bm{d}}^k_{G_i}\|_2^2 = \max_{i \in \A^k}\|{\bm{\beta}}^k_{G_i}\|_2^2 \geqslant \max_{i \in \mS^k_{C, 1}}\|{\bm{\beta}}^k_{G_i}\|_2^2.
\end{align*} 
By simple algebra, the corresponding range of $\kappa$ is
\begin{equation}\label{s2.7}
\kappa \leqslant \frac{\min_{j \in \mS^k_{C, 2}}\|{\bm{d}}^k_{G_j}\|_2}{\max_{i\in \mS^k_{C, 1}}\|{\bm{\beta}}^k_{G_i}\|_2}.
\end{equation}
Similar to \eqref{s2.7}, given $C+1$, we obtain that the range of $\kappa$ is
\begin{equation}\label{s2.8}
\kappa \leqslant \frac{\min_{j \in \mS^k_{C+1, 2}}\|{\bm{d}}^k_{G_j}\|_2}{\max_{i\in \mS^k_{C+1, 1}}\|{\bm{\beta}}^k_{G_i}\|_2}.
\end{equation}
Note that $\mS^k_{C, 1} \subseteq \mS^k_{C+1, 1}$ and $\mS^k_{C, 2} \subseteq \mS^k_{C+1, 2}$.
Given $C$, $\kappa$ takes values in the difference set between \eqref{s2.7} and \eqref{s2.8}.
Therefore, given $C<|\A^k|$, the corresponding range of $\kappa$ is
\begin{equation*}
\kappa \in \left(\frac{\min_{j \in \mS^k_{C+1, 2}}\|{\bm{d}}^k_{G_j}\|_2}{\max_{i\in \mS^k_{C+1, 1}}\|{\bm{\beta}}^k_{G_i}\|_2}, \frac{\min_{j \in \mS^k_{C, 2}}\|{\bm{d}}^k_{G_j}\|_2}{\max_{i\in \mS^k_{C, 1}}\|{\bm{\beta}}^k_{G_i}\|_2}\right].
\end{equation*}
Given $C = |\A^k|$, note that $\mS^k_{C, 1}= \A^k$. We have
\begin{align*}
  \min_{j \in \mS^k_{C, 2}} \frac{1}{\kappa^2}\|{\bm{d}}^k_{G_j}\|_2^2 = \min_{j \in \mS^k_{C, 2}} \|{\bm{\beta}}^k_{G_j}+\frac{1}{\kappa}{\bm{d}}^k_{G_j}\|_2^2 \geqslant \max_{i \in \A^k}\|{\bm{\beta}}^k_{G_i}+\frac{1}{\kappa}{\bm{d}}^k_{G_i}\|_2^2 = \max_{i \in \A^k}\|{\bm{\beta}}^k_{G_i}\|_2^2.
\end{align*} 
Therefore, by simple algebra,
the corresponding range of $\kappa$ is
\begin{equation*}
  \kappa \in \left(0, \frac{\min_{j \in \mS^k_{C, 2}}\|{\bm{d}}^k_{G_j}\|_2}{\max_{i\in \A^k}\|{\bm{\beta}}^k_{G_i}\|_2}\right].
\end{equation*}
\qedsymbol
\end{proof}

\subsection{Auxiliary Lemmas}
To simplify the proofs of main theorems, we provide three useful lemmas. Lemma \ref{lem:C} provides some valuable inequalities occurring frequently. In Lemma \ref{lem:A}, we prove the upper bound of $\|{\bm{\beta}}^*_{\A_{12}}\|_2$ and $\|{\bm{\beta}}^*_{\I_{12}}\|_2$ in terms of $\|{\bm{\beta}}^*_{\I_1}\|_2$ with high probability.
In Lemma \ref{lem:B}, we show that the components related to ${\bm{\varepsilon}}$ can be controlled by $n$ and $\|{\bm{\beta}}^*_{\I_1}\|_2$ with high probability. 

Denote the exchanged subsets of groups in the selected set and unselected set with size $ C = |\I_1|$, respectively, as
\begin{align*}
  &\mS_1 = \{j \in \hat\A:\sum_{i\in \hat \A}I(\|\hat {\bm{\beta}}_{G_j}\|_2^2 \geqslant \|\hat {\bm{\beta}}_{G_i}\|_2^2)\leqslant C\}, \\
  &\mS_2 = \{j \in \hat\I:\sum_{i\in \hat \I}I(\|\hat {\bm{d}}_{G_j}\|_2^2 \leqslant \|\hat {\bm{d}}_{G_i}\|_2^2) \leqslant C\}.
\end{align*}
Here $\mS_1$ represents the subset of groups exchanging from $\hat \A$ to $\hat \I$, and $\mS_2$ represents the subset of groups exchanging from $\hat \I$ to $\hat \A$.
Let
\begin{align*}
  &\A_1 = \hat \A \cap \A^*,\ \A_2=\hat \A\cap \I^*,\\
  &\I_1 = \hat \I \cap \A^*,\ \I_2 = \hat \I \cap \I^*.
\end{align*}
Denote the subset of groups preserving or exchanging in the selected set as
\vspace{-0.2cm}
\begin{align*}
  &\A_{11} = \A_1 \backslash \mS_1,\ \A_{21} = \A_2 \backslash \mS_1 ,\\
  &\A_{12}= \A_1\cap \mS_1,\ \A_{22} = \A_2 \cap \mS_1.
\end{align*}
And denote the subset of groups exchanging or preserving in the unselected set as
\vspace{-0.3cm}
\begin{align*}
  &\I_{11} = \I_1 \cap \mS_2,\ \I_{21} = \I_2 \cap \mS_2,\\
  &\I_{12}=\I_1\backslash \mS_2,\ \I_{22} = \I_2 \backslash \mS_2.
\end{align*}
Let $\tilde{\A} = (\hat \A\backslash \mS_1)\cup \mS_2$ and $\tilde{\I} = (\hat \I\backslash \mS_2)\cup \mS_1$ be the selected set and unselected set after group splicing, and ${\bm{H}}_{\A} = {\bm{X}}_{\A} ({\bm{X}}_{\A}^\top {\bm{X}}_{\A})^{-1}{\bm{X}}_{\A}^\top$ be the hat matrix and ${\bm{X}}_{G_j}^{(i)}$ be the $i$th column of sub-matrix ${\bm{X}}_{G_j}$.
Denote the least-squares estimator contained on $\cup_{j \in \A}G_j$ as $\hat {\bm{\beta}}_{\A}$.
Given $\A \subseteq \mB$, denote ${\bm{e}}_{\A}{\bm{u}}_{\mB}$ as the vector ${\bm{u}}_{\mB}$ contained in $\cup_{j \in \A} G_j$, where ${\bm{e}}_{\A} \in \mathbb{R}^{\#\{\mB\}}$ supports on $\cup_{j\in \A}G_j$ with all nonzero elements equal to one.

\begin{lemma}\label{lem:C}
  Let $\A$ and $\mB$ be disjoint subsets of groups of $\mS$ with  $|\A| \leqslant \tau$ and $|\mB| \leqslant \tau$. Assume ${\bm{X}}$ satisfies GSRC with order $2\tau$. Then for any ${\bm{u}} \in \mathbb{R}^{\#\{\A\}}$,  we have
  \begin{equation}\label{lemC3}
    n\left(c_*(\tau)-\frac{\omega_{\tau}^2}{c_*(\tau)}\right)\|{\bm{u}}\|_2 \leqslant \|{\bm{X}}_{\A}^\top(\bm{\mathrm{I}} - {\bm{H}}_{\mB}) {\bm{X}}_{\A}{\bm{u}}\|_2 \leqslant n\left(c^*(\tau)+\frac{\omega_{\tau}^2}{c_*(\tau)}\right)\|{\bm{u}}\|_2,
  \end{equation}
  \begin{equation}\label{lemC5}
    \frac{\|{\bm{u}}\|_2}{n\left(c^*(\tau)+\frac{\omega_{\tau}^2}{c_*(\tau)}\right)} \leqslant \|({\bm{X}}_{\A}^\top(\bm{\mathrm{I}} - {\bm{H}}_{\mB}) {\bm{X}}_{\A})^{-1}{\bm{u}}\|_2 \leqslant \frac{\|{\bm{u}}\|_2}{n\left(c_*(\tau)-\frac{\omega_{\tau}^2}{c_*(\tau)}\right)},
  \end{equation}
  where $c_*(\tau), c^*(\tau)$ and $\omega_{\tau}$ are defined in (C2).
  Additionally, $c_*(\tau)$ decreases while $c^*(\tau)$ increases as $\tau$ increases. $\omega_{\tau}$ is bounded by $\delta_{T} = \max\{1-c_*(2T), c^*(2T)-1\}$.
\end{lemma}
\begin{proof}

  Note that ${\bm{X}}_{\A}^\top(\bm{\mathrm{I}}_n - {\bm{H}}_{\mB}) {\bm{X}}_{\A} = {\bm{X}}_{\A}^\top {\bm{X}}_{\A} - {\bm{X}}_{\A}^\top {\bm{X}}_{\mB}({\bm{X}}_{\mB}^\top {\bm{X}}_{\mB})^{-1}{\bm{X}}_{\mB}^\top {\bm{X}}_{\A}$.
  For the right-hand side of \eqref{lemC3}, we have
  \begin{align*}
    \|{\bm{X}}_{\A}^\top(\bm{\mathrm{I}}_n - {\bm{H}}_{\mB}) {\bm{X}}_{\A}\bm u\|_2 \leqslant& \|{\bm{X}}_{\A}^\top {\bm{X}}_{\A}{\bm{u}}\|_2 + \|{\bm{X}}_{\A}^\top {\bm{X}}_{\mB}({\bm{X}}_{\mB}^\top {\bm{X}}_{\mB})^{-1}{\bm{X}}_{\mB}^\top {\bm{X}}_{\A}{\bm{u}}\|_2\\
    \leqslant& nc^*(\tau)\|{\bm{u}}\|_2 + n\frac{\omega_{\tau}^2}{c_*(\tau)}\|{\bm{u}}\|_2\\
    =&n\left(c^*(\tau)+\frac{\omega_{\tau}^2}{c_*(\tau)}\right)\|{\bm{u}}\|_2,
  \end{align*}
  where the first inequality follows from triangle inequality, and the second inequality follows from the definition of $\omega_{\tau}$.
  Similarly, for the left-hand side, we have
  \begin{align*}
    \|{\bm{X}}_{\A}^\top(\bm{\mathrm{I}}_n - {\bm{H}}_{\mB}) {\bm{X}}_{\A}{\bm{u}}\|_2 \geqslant& \|{\bm{X}}_{\A}^\top {\bm{X}}_{\A}{\bm{u}}\|_2 - \|{\bm{X}}_{\A}^\top {\bm{X}}_{\mB}({\bm{X}}_{\mB}^\top {\bm{X}}_{\mB})^{-1}{\bm{X}}_{\mB}^\top {\bm{X}}_{\A}{\bm{u}}\|_2\\
    \geqslant& nc_*(\tau)\|{\bm{u}}\|_2 - n\frac{\omega_{\tau}^2}{c_*(\tau)}\|{\bm{u}}\|_2\\
    =&n\left(c_*(\tau)-\frac{\omega_{\tau}^2}{c_*(\tau)}\right)\|{\bm{u}}\|_2.
  \end{align*}
  This proves \eqref{lemC3}. \eqref{lemC5} is a direct consequence of \eqref{lemC3}.  
  Obviously, when $\tau$ increases, $c_*(\tau)$ decreases and $c^*(\tau)$ increases. The spectrum of ${\bm{X}}_{\A}^\top {\bm{X}}_{\mB}/n$
  can be bounded by $\delta_{T} = \max\{1-c_*(2T), c^*(2T)-1\}$ since ${\bm{X}}_{\A}^\top {\bm{X}}_{\mB}/n$ is an off-diagonal sub-matrix of ${\bm{X}}_{\A\cup \mB}^\top {\bm{X}}_{\A\cup\mB}/n-\bm{\mathrm{I}}_{\#\{\A\cup\mB\}}$.
  \qedsymbol
\end{proof}

\begin{lemma}\label{lem:A}
  Assume the conditions in Theorem 1 hold. With probability at least $1-\delta_1$, we have
  \begin{equation}\label{ap:Le1}
    \|{\bm{\beta}}^*_{\A_{12}}\|_2 \leqslant 2(1+\eta)\frac{\omega_T}{c_*(T)}\|{\bm{\beta}}^*_{\I_1}\|_2
  \end{equation}
  and
  \begin{equation}\label{ap:Le2}
    \|{\bm{\beta}}^*_{\I_{12}}\|_2 \leqslant 2(1+\eta)\frac{\left(\omega_T+\frac{\omega_T^2}{c_*(T)}\right)}{c_*(T)}\|{\bm{\beta}}^*_{\I_1}\|_2.
  \end{equation}
\end{lemma}
\begin{proof}
  By the definition of $\mS_1$, we have
  \begin{equation*}
    \frac{1}{|\A_{12}|}\sum_{j \in \A_{12}}  \|\hat{\bm{\beta}}_{G_j}\|_2^2 \leqslant \frac{1}{|\A_{21}|}\sum_{j \in \A_{21}}  \|\hat{\bm{\beta}}_{G_j}\|_2^2.
  \end{equation*}
  Then,
  \begin{equation}\label{eq:lemA11}
    \|\hat{\bm{\beta}}_{\A_{12}}\|_2 \leqslant \sqrt{\frac{|\A_{12}|}{|\A_{21}|}}\|\hat{\bm{\beta}}_{\A_{21}}\|_2.
  \end{equation}
  Note that the maximum profile likelihood estimators of ${\bm{\beta}}_{\A_1}^*$ and ${\bm{\beta}}_{\A_2}^*$ are
  \begin{align}\label{eq:lemA1}
    \begin{split}
    \hat {\bm{\beta}}_{\A_1} &=({\bm{X}}^\top_{\A_1}(\bm{\mathrm{I}}_n-{\bm{H}}_{\A_2}){\bm{X}}_{\A_1})^{-1}{\bm{X}}_{\A_1}^\top(\bm{\mathrm{I}}_n-{\bm{H}}_{\A_2}){\bm{y}}\\
    &= {\bm{\beta}}^*_{\A_1} + ({\bm{X}}^\top_{\A_1}(\bm{\mathrm{I}}_n-{\bm{H}}_{\A_2}){\bm{X}}_{\A_1})^{-1}{\bm{X}}_{\A_1}^\top(\bm{\mathrm{I}}_n-{\bm{H}}_{\A_2})({\bm{X}}_{\I_1}{\bm{\beta}}^*_{\I_1}+{\bm{\varepsilon}})
  \end{split}
  \end{align}
  and
  \begin{align}\label{eq:lemA2}
    \begin{split}
    \hat {\bm{\beta}}_{\A_2} &=({\bm{X}}^\top_{\A_2}(\bm{\mathrm{I}}_n-{\bm{H}}_{\A_1}){\bm{X}}_{\A_2})^{-1}{\bm{X}}_{\A_2}^\top(\bm{\mathrm{I}}_n-{\bm{H}}_{\A_1}){\bm{y}}\\
    &= {\bm{\beta}}^*_{\A_2} + ({\bm{X}}^\top_{\A_2}(\bm{\mathrm{I}}_n-{\bm{H}}_{\A_1}){\bm{X}}_{\A_2})^{-1}{\bm{X}}_{\A_2}^\top(\bm{\mathrm{I}}_n-{\bm{H}}_{\A_1})({\bm{X}}_{\I_1}{\bm{\beta}}^*_{\I_1}+{\bm{\varepsilon}})\\
    &=({\bm{X}}^\top_{\A_2}(\bm{\mathrm{I}}_n-{\bm{H}}_{\A_1}){\bm{X}}_{\A_2})^{-1}{\bm{X}}_{\A_2}^\top(\bm{\mathrm{I}}_n-{\bm{H}}_{\A_1})({\bm{X}}_{\I_1}{\bm{\beta}}^*_{\I_1}+{\bm{\varepsilon}}),
  \end{split}
  \end{align}
  where the second equalities in \eqref{eq:lemA1} and \eqref{eq:lemA2} use the projection property of $\bm{\mathrm{I}}_n-{\bm{H}}_{\A_2}$ and $\bm{\mathrm{I}}_n-{\bm{H}}_{\A_1}$, and the last equality in \eqref{eq:lemA2} follows from ${\bm{\beta}}_{\A_2}^*=0$.
  By \eqref{eq:lemA1}, we have
  \begin{align}\label{eq:lemA3}
    \begin{split}
    \|\hat {\bm{\beta}}_{\A_{12}}\|_2 \geqslant& \|{\bm{\beta}}^*_{\A_{12}}\|_2 - \|{\bm{e}}^\top_{\A_{12}}({\bm{X}}^\top_{\A_1}(\bm{\mathrm{I}}_n-{\bm{H}}_{\A_2}){\bm{X}}_{\A_1})^{-1}{\bm{X}}_{\A_1}^\top(\bm{\mathrm{I}}_n-{\bm{H}}_{\A_2}){\bm{X}}_{\I_1}{\bm{\beta}}^*_{\I_1}\|_2-\\
    &\|{\bm{e}}^\top_{\A_{12}}({\bm{X}}^\top_{\A_1}(\bm{\mathrm{I}}_n-{\bm{H}}_{\A_2}){\bm{X}}_{\A_1})^{-1}{\bm{X}}_{\A_1}^\top(\bm{\mathrm{I}}_n-{\bm{H}}_{\A_2}){\bm{\varepsilon}}\|_2\\
    \geqslant& \|{\bm{\beta}}^*_{\A_{12}}\|_2 -\frac{\omega_T}{c_*(T)}\|{\bm{\beta}}^*_{\I_1}\|_2-\\
    &\|{\bm{e}}^\top_{\A_{12}}({\bm{X}}^\top_{\A_{1}}(\bm{\mathrm{I}}_n-{\bm{H}}_{\A_2}){\bm{X}}_{\A_1})^{-1}{\bm{X}}_{\A_1}^\top (\bm{\mathrm{I}}_n-{\bm{H}}_{\A_2}){\bm{\varepsilon}}\|_2,
  \end{split}
  \end{align}
  where the first inequality follows from the triangle inequality, and the second inequality follows from \eqref{lemC5} and the definition of $\omega_T$.
  Similarly, by \eqref{eq:lemA2}, we have
\begin{align}\label{eq:lemA4}
  \begin{split}
  \|\hat {\bm{\beta}}_{\A_{21}}\|_2 \leqslant& \|{\bm{e}}^\top_{\A_{21}}({\bm{X}}^\top_{A_2}(\bm{\mathrm{I}}_n-{\bm{H}}_{\A_1}){\bm{X}}_{\A_2})^{-1}{\bm{X}}_{\A_2}^\top(\bm{\mathrm{I}}_n-{\bm{H}}_{\A_1}){\bm{X}}_{\I_1}{\bm{\beta}}^*_{\I_1}\|_2+\\
  &\|{\bm{e}}^\top_{\A_{21}}({\bm{X}}^\top_{A_2}(\bm{\mathrm{I}}_n-{\bm{H}}_{\A_1}){\bm{X}}_{\A_2})^{-1}{\bm{X}}_{\A_2}^\top(\bm{\mathrm{I}}_n-{\bm{H}}_{\A_1}){\bm{\varepsilon}}\|_2\\
  \leqslant& \frac{\omega_T}{c_*(T)}\|{\bm{\beta}}^*_{\I_1}\|_2+\|{\bm{e}}^\top_{\A_{21}}({\bm{X}}^\top_{\A_{2}}(\bm{\mathrm{I}}_n-{\bm{H}}_{\A_1}){\bm{X}}_{\A_2})^{-1}{\bm{X}}_{\A_2}^\top (\bm{\mathrm{I}}_n-{\bm{H}}_{\A_1}){\bm{\varepsilon}}\|_2.
\end{split}
\end{align}
From \eqref{eq:lemA11}, \eqref{eq:lemA3} and \eqref{eq:lemA4}, we have
\begin{align}\label{lemA5}
  \begin{split}
  \|{\bm{\beta}}^*_{\A_{12}}\|_2 \leqslant &\left(1+\sqrt{\frac{|\A_{12}|}{|\A_{21}|}}\right)\frac{\omega_T}{c_*(T)}\|{\bm{\beta}}^*_{\I_1}\|_2+\\
  &\|{\bm{e}}^\top_{\A_{12}}({\bm{X}}^\top_{\A_{1}}(\bm{\mathrm{I}}_n-{\bm{H}}_{\A_2}){\bm{X}}_{\A_1})^{-1}{\bm{X}}_{\A_1}^\top (\bm{\mathrm{I}}_n-{\bm{H}}_{\A_2}){\bm{\varepsilon}}\|_2+\\
  &\sqrt{\frac{|\A_{12}|}{|\A_{21}|}}\|{\bm{e}}^\top_{\A_{21}}({\bm{X}}^\top_{A_2}(\bm{\mathrm{I}}_n-{\bm{H}}_{\A_1}){\bm{X}}_{\A_2})^{-1}{\bm{X}}_{\A_2}^\top(\bm{\mathrm{I}}_n-{\bm{H}}_{\A_1}){\bm{\varepsilon}}\|_2.
\end{split}
\end{align}
Since $\mS_1 = \A_{12} \cup \A_{22}$ and $\A_2 = \A_{21} \cup\A_{22}$.
Note that $|\I_1| = |\mS_1| = C$ and $|\I_1|+|\A_1| = |\A^*| \leqslant |\hat \A| = |\A_1| + |\A_2|$. We have $|\A_{12}| \leqslant |\A_{21}|$.
Therefore, we can simplify \eqref{lemA5} as
\begin{align}\label{lemA6}
  \begin{split}
  \|{\bm{\beta}}^*_{\A_{12}}\|_2 \leqslant &2\frac{\omega_T}{c_*(T)}\|{\bm{\beta}}^*_{\I_1}\|_2+\|{\bm{e}}^\top_{\A_{12}}({\bm{X}}^\top_{\A_{1}}(\bm{\mathrm{I}}_n-{\bm{H}}_{\A_2}){\bm{X}}_{\A_1})^{-1}{\bm{X}}_{\A_1}^\top (\bm{\mathrm{I}}_n-{\bm{H}}_{\A_2}){\bm{\varepsilon}}\|_2+\\
  &\|{\bm{e}}^\top_{\A_{21}}({\bm{X}}^\top_{A_2}(\bm{\mathrm{I}}_n-{\bm{H}}_{\A_1}){\bm{X}}_{\A_2})^{-1}{\bm{X}}_{\A_2}^\top(\bm{\mathrm{I}}_n-{\bm{H}}_{\A_1}){\bm{\varepsilon}}\|_2.
\end{split}
\end{align}
Next, we bound the components related to ${\bm{\varepsilon}}$ in terms of $\|{\bm{\beta}}^*_{\I_1}\|_2$. We have
\begin{align}\label{lemA7}
  \begin{split}
  &P\left(\|{\bm{e}}^\top_{\A_{12}}({\bm{X}}^\top_{\A_{1}}(\bm{\mathrm{I}}_n-{\bm{H}}_{\A_2}){\bm{X}}_{\A_1})^{-1}{\bm{X}}_{\A_1}^\top (\bm{\mathrm{I}}_n-{\bm{H}}_{\A_2}){\bm{\varepsilon}}\|_2>\eta\frac{\omega_T}{c_*(T)}\|{\bm{\beta}}^*_{\I_1}\|_2\right)\\
\leqslant & P\left(\|{\bm{X}}^\top_{\A_1}(\bm{\mathrm{I}}_n-{\bm{H}}_{\A_2}){\bm{\varepsilon}}\|_2>\eta\frac{n\omega_T\left(c_*(T)-\frac{\omega_T^2}{c_*(T)}\right)}{c_*(T)}\|{\bm{\beta}}^*_{\I_1}\|_2\right)\\
\leqslant& \sum_{j\in \A_1}\sum_{i=1}^{p_j}P\left(|({\bm{X}}^{(i)}_{G_j})^\top{\bm{\varepsilon}}| > \frac{n\eta\omega_T\left(c_*(T)-\frac{\omega_T^2}{c_*(T)}\right)}{c_*(T)\sqrt{\#\{\A_1\}}}\|{\bm{\beta}}^*_{\I_1}\|_2\right)\\
\leqslant& 2p \exp\{-nC_1\vartheta /s^*p_{\max}\} \leqslant \frac{\delta_1}{2},
\end{split}
\end{align}
where the first inequality follows from the right-hand side of \eqref{lemC5}, the second inequality follows from the idempotency of $\bm{\mathrm{I}}_n-{\bm{H}}_{\A_1}$, and the third inequality follows from the Hoeffding's inequality and $\#\{\A_1\} \leqslant s^*p_{\max}$ in which $C_1$ is some positive constant depending on the spectrum bounds in GSRC.
Similar to \eqref{lemA7}, we have
\begin{align}\label{lemA8}
  \begin{split}
  &P\left(\|{\bm{e}}^\top_{\A_{21}}({\bm{X}}^\top_{\A_{2}}(\bm{\mathrm{I}}_n-{\bm{H}}_{\A_1}){\bm{X}}_{\A_2})^{-1}{\bm{X}}_{\A_2}^\top (\bm{\mathrm{I}}_n-{\bm{H}}_{\A_1}){\bm{\varepsilon}}\|_2>\eta\frac{\omega_T}{c_*(T)}\|{\bm{\beta}}^*_{\I_1}\|_2\right)\\
\leqslant& 2p \exp\{-nC_1\vartheta /s^*p_{\max}\} \leqslant \frac{\delta_1}{2}.
\end{split}
\end{align}
Combini8ng \eqref{lemA6}, \eqref{lemA7} and \eqref{lemA8}, we have
\begin{align*}
  P\left(\|{\bm{\beta}}^*_{\A_{12}}\|_2 \leqslant 2(1+\eta)\frac{\omega_T}{c_*(T)}\|{\bm{\beta}}^*_{\I_1}\|_2\right) \geqslant 1-\delta_1.
  \end{align*}

Next, we turn to the proof of \eqref{ap:Le2}.
By the definition of $\mS_2$, we have
\begin{equation*}
  \frac{1}{|\I_{12}|}\sum_{j \in \I_{12}}  \|\hat d_{G_j}\|_2^2 \leqslant \frac{1}{|\I_{21}|}\sum_{j \in \I_{21}}  \|\hat d_{G_j}\|_2^2.
\end{equation*}
Since $|\I_{11}|+|\I_{21}|=|\mS_2|=|\I_1|=|\I_{11}|+|\I_{12}| = C$, we have $|\I_{21}| = |\I_{12}|$.
Then we obtain
\begin{equation}\label{lemA12}
  \|\hat d_{\I_{12}}\|_2 \leqslant \|\hat d_{\I_{21}}\|_2.
\end{equation}
Note that
\begin{align}\label{lemA9}
  \begin{split}
  n\|\hat d_{\I_{12}}\|_2 =& \|{\bm{X}}_{\I_{12}}^\top(\bm{\mathrm{I}}_n-{\bm{H}}_{\hat \A}){\bm{y}}\|_2=\|{\bm{X}}_{\I_{12}}^\top(\bm{\mathrm{I}}_n-{\bm{H}}_{\hat \A})({\bm{X}}_{\I_1}{\bm{\beta}}^*_{\I_1}+{\bm{\varepsilon}})\|_2\\
  \geqslant& \|{\bm{X}}_{\I_{12}}^\top(\bm{\mathrm{I}}_n-{\bm{H}}_{\hat \A}){\bm{X}}_{\I_{12}}{\bm{\beta}}^*_{\I_{12}}\|_2 - \|{\bm{X}}_{\I_{12}}^\top(\bm{\mathrm{I}}_n-{\bm{H}}_{\hat \A}){\bm{X}}_{\I_{11}}{\bm{\beta}}^*_{\I_{11}}\|_2-\\
  &\|{\bm{X}}_{\I_{12}}^\top(\bm{\mathrm{I}}_n-{\bm{H}}_{\hat \A}){\bm{\varepsilon}}\|_2\\
  \geqslant& n\left(c_*(T)-\frac{\omega_T^2}{c_*(T)}\right)\|{\bm{\beta}}^*_{\I_{12}}\|_2 - n\omega_T \|{\bm{\beta}}^*_{\I_{11}}\|_2 - \|{\bm{X}}_{\I_{12}}^\top(\bm{\mathrm{I}}_n-{\bm{H}}_{\hat \A}){\bm{\varepsilon}}\|_2\\
  \geqslant& nc_*(T)\|{\bm{\beta}}^*_{\I_{12}}\|_2 - n\left(\omega_T+\frac{\omega_T^2}{c_*(T)}\right) \|{\bm{\beta}}^*_{\I_1}\|_2 - \|{\bm{X}}_{\I_{12}}^\top(\bm{\mathrm{I}}_n-{\bm{H}}_{\hat \A}){\bm{\varepsilon}}\|_2,
  \end{split}
\end{align}
where the second inequality follows from \eqref{lemC3} and the definition of $\omega_T$.
Similarly, we have
\begin{align}\label{lemA10}
  \begin{split}
  n\|\hat d_{\I_{21}}\|_2 =& \|{\bm{X}}_{\I_{21}}^\top(\bm{\mathrm{I}}_n-{\bm{H}}_{\hat \A}){\bm{y}}\|_2=\|{\bm{X}}_{\I_{21}}^\top(\bm{\mathrm{I}}_n-{\bm{H}}_{\hat \A})({\bm{X}}_{\I_1}{\bm{\beta}}^*_{\I_1}+{\bm{\varepsilon}})\|_2\\
  \leqslant& \|{\bm{X}}_{\I_{21}}^\top(\bm{\mathrm{I}}_n-{\bm{H}}_{\hat \A}){\bm{X}}_{\I_{1}}{\bm{\beta}}^*_{\I_{1}}\|_2 +\|{\bm{X}}_{\I_{12}}^\top(\bm{\mathrm{I}}_n-{\bm{H}}_{\hat \A}){\bm{\varepsilon}}\|_2\\
  \leqslant& n\left(\omega_T+\frac{\omega_T^2}{c_*(T)}\right)\|{\bm{\beta}}^*_{\I_{1}}\|_2 + \|{\bm{X}}_{\I_{21}}^\top(\bm{\mathrm{I}}_n-{\bm{H}}_{\hat \A}){\bm{\varepsilon}}\|_2.
\end{split}
\end{align}
Combine\eqref{lemA12}, \eqref{lemA9} and \eqref{lemA10},
\begin{align}\label{lemA13}
  \begin{split}
  &2n\left(\omega_T+\frac{\omega_T^2}{c_*(T)}\right)\|{\bm{\beta}}_{\I_1}^*\|_2+\|{\bm{X}}_{\I_{21}}^\top(\bm{\mathrm{I}}_n-{\bm{H}}_{\hat \A}){\bm{\varepsilon}}\|_2\\
  \geqslant& nc_*(T)\|{\bm{\beta}}^*_{\I_{12}}\|_2 -\|{\bm{X}}_{\I_{12}}^\top(\bm{\mathrm{I}}_n-{\bm{H}}_{\hat \A}){\bm{\varepsilon}}\|_2.
  \end{split}
\end{align}
Note that
\begin{align}\label{lemA14}
  \begin{split}
  &P\left(\|{\bm{X}}_{\I_{21}}^\top(\bm{\mathrm{I}}_n-{\bm{H}}_{\hat \A}){\bm{\varepsilon}}\|_2 \geqslant \eta n\left(\omega_T+\frac{\omega_T^2}{c_*(T)}\right)\|{\bm{\beta}}_{\I_1}^*\|_2\right)\\
\leqslant& \sum_{j\in \I_{21}}\sum_{i=1}^{p_j}P\left(|({\bm{X}}^{(i)}_{G_j})^\top{\bm{\varepsilon}}| > \frac{n\eta \left(\omega_T+\frac{\omega_T^2}{c_*(T)}\right)}{\sqrt{\#\{\I_{21}\}}}\|{\bm{\beta}}^*_{\I_1}\|_2\right)\\
\leqslant& 2p \exp\{-nC_1\vartheta /s^*p_{\max}\} \leqslant \frac{\delta_1}{2},
  \end{split}
\end{align}
where the second inequality follows from $\#\{\I_{21}\} \leqslant s^*p_{\max}$. Here positive constant $C_1$ depends on the spectrum bounds in GSRC.
Similarly, we have,
\begin{align}\label{lemA15}
  \begin{split}
  P\left(\|{\bm{X}}_{\I_{12}}^\top(\bm{\mathrm{I}}_n-{\bm{H}}_{\hat \A}){\bm{\varepsilon}}\|_2 \geqslant \eta n\left(\omega_T+\frac{\omega_T^2}{c_*(T)}\right)\|{\bm{\beta}}_{\I_1}^*\|_2\right)\leqslant \frac{\delta_1}{2}.
  \end{split}
\end{align}
Combine \eqref{lemA13}, \eqref{lemA14} and \eqref{lemA15},
\begin{align*}
P\left(\|{\bm{\beta}}^*_{\I_{12}}\|_2 \leqslant 2(1+\eta)\frac{\left(\omega_T+\frac{\omega_T^2}{c_*(T)}\right)}{c_*(T)}\|{\bm{\beta}}^*_{\I_1}\|_2\right) \geqslant 1-\delta_1.
\end{align*}
This proves \eqref{ap:Le2}, which completes the proof of Lemma~\ref{lem:A}.
\qedsymbol
\end{proof}

\begin{lemma}\label{lem:B}
  Assume the conditions in Theorem 1 hold. We have
  \begin{equation}\label{ap:Le5}
    |({\bm{X}}_{\I_{1}}{\bm{\beta}}^*_{\I_{1}})^\top (\bm{\mathrm{I}}_n-{\bm{H}}_{\hat{\A}}){\bm{\varepsilon}}| \leqslant \frac{\eta}{8}n\left(c_*(T)-\frac{\omega_T^2}{c_*(T)}\right)\|{\bm{\beta}}^*_{\I_1}\|_2^2,
  \end{equation}
  \begin{equation}\label{ap:Le6}
    |({\bm{X}}_{\A_{12}\cup \I_{12}}{\bm{\beta}}^*_{\A_{12}\cup \I_{12}})^\top (\bm{\mathrm{I}}_n-{\bm{H}}_{\tilde{\A}}){\bm{\varepsilon}}| \leqslant \frac{\eta}{8}n\left(c_*(T)-\frac{\omega_T^2}{c_*(T)}\right)\|{\bm{\beta}}^*_{\I_1}\|_2^2,
  \end{equation}
  with probability at least $1-\delta_1$, and
  \begin{equation}\label{ap:Le4}
    \|{\bm{H}}_{\hat{\A}}{\bm{\varepsilon}}\|_2 \leqslant \sqrt{\frac{\eta}{4}n\left(c_*(T)-\frac{\omega_T^2}{c_*(T)}\right)\|{\bm{\beta}}^*_{\I_1}\|_2^2}
  \end{equation}
  \begin{equation}\label{ap:Le3}
    \|{\bm{H}}_{\tilde{\A}}{\bm{\varepsilon}}\|_2 \leqslant \sqrt{\frac{\eta}{4}n\left(c_*(T)-\frac{\omega_T^2}{c_*(T)}\right)\|{\bm{\beta}}^*_{\I_1}\|_2^2}
  \end{equation}
  with probability at least $1-\delta_2$.
\end{lemma}
\begin{proof}
  Firstly, we show ~\eqref{ap:Le5}.
\begin{align*}
  &P\left(|({\bm{X}}_{\I_{1}}{\bm{\beta}}^*_{\I_{1}})^\top (\bm{\mathrm{I}}_n-{\bm{H}}_{\tilde{\A}}){\bm{\varepsilon}}|> \frac{\eta}{8}n\left(c_*(T)-\frac{\omega_T^2}{c_*(T)}\right)\|{\bm{\beta}}^*_{\I_1}\|_2^2\right) \\
\leqslant &P\left(\|{\bm{X}}_{\I_1}^\top{\bm{\varepsilon}}\|_2 \|{\bm{\beta}}^*_{\I_1}\|_2> \frac{\eta}{8}n\left(c_*(T)-\frac{\omega_T^2}{c_*(T)}\right)\|{\bm{\beta}}^*_{\I_1}\|_2^2\right)\\
\leqslant & \sum_{j\in \I_1}\sum_{i=1}^{p_j}P\left(|({\bm{X}}^{(i)}_{G_j})^\top{\bm{\varepsilon}}|>\frac{\eta}{8\sqrt{\#\{\I_1\}}}n\left(c_*(T)-\frac{\omega_T^2}{c_*(T)}\right)\|{\bm{\beta}}^*_{\I_1}\|_2\right)\\
\leqslant &2p \exp\{-nC_1\vartheta /s^*p_{\max}\} \leqslant \delta_1,
\end{align*}
where the first inequality uses Cauchy inequality, and $C_1$ is some positive constant depending on the spectrum bounds in GSRC.

Next we prove \eqref{ap:Le6}. With probability at least $1-\delta_1$, by \eqref{ap:Le1} and \eqref{ap:Le2}, we have
\begin{align}\label{lemB1}
  \begin{split}
    \|{\bm{\beta}}^*_{\A_{12}\cup \I_{12}}\|_2^2 \leqslant&(2(1+\eta)\frac{\omega_T}{c_*(T)})^2\|{\bm{\beta}}^*_{\I_1}\|_2^2+(2(1+\eta)\frac{\left(\omega_T+\frac{\omega_T^2}{c_*(T)}\right)}{c_*(T)})^2\|{\bm{\beta}}^*_{\I_1}\|_2^2\\
    \leqslant&8(\frac{(1+\eta)\left(\omega_T+\frac{\omega_T^2}{c_*(T)}\right)}{c_*(T)})^2\|{\bm{\beta}}^*_{\I_1}\|_2^2.
  \end{split}
\end{align}
Thus, we have
\begin{align*}
  &P\left(|({\bm{X}}_{\A_{12}\cup \I_{12}}{\bm{\beta}}^*_{\A_{12}\cup \I_{12}})^\top (\bm{\mathrm{I}}_n-{\bm{H}}_{\tilde{\A}}){\bm{\varepsilon}}| > \frac{\eta}{8}n\left(c_*(T)-\frac{\omega_T^2}{c_*(T)}\right)\|{\bm{\beta}}^*_{\I_1}\|_2^2\right) \\
\leqslant &P\left(\|{\bm{X}}_{\A_{12}\cup \I_{12}}^\top{\bm{\varepsilon}}\|_2 \|{\bm{\beta}}^*_{\A_{12}\cup \I_{12}}\|_2> \frac{\eta}{8}n\left(c_*(T)-\frac{\omega_T^2}{c_*(T)}\right)\|{\bm{\beta}}^*_{\I_1}\|_2^2\right)\\
\leqslant &P\left(\|{\bm{X}}_{\A_{12}\cup \I_{12}}^\top{\bm{\varepsilon}}\|_2 (2\sqrt{2}\frac{(1+\eta)\left(\omega_T+\frac{\omega_T^2}{c_*(T)}\right)}{c_*(T)})\|{\bm{\beta}}^*_{\I_1}\|_2> \frac{\eta}{8}n\left(c_*(T)-\frac{\omega_T^2}{c_*(T)}\right)\|{\bm{\beta}}^*_{\I_1}\|_2^2\right)\\
\leqslant & \sum_{j\in \A_{12}\cup \I_{12}}\sum_{i=1}^{p_j}P\left(|({\bm{X}}^{(i)}_{G_j})^\top{\bm{\varepsilon}}|>\frac{\eta nc_*(T)\left(c_*(T)-\frac{\omega_T^2}{c_*(T)}\right)}{16\sqrt{2}(1+\eta)(\omega_T+\frac{\omega_T^2}{c_*(T)}t) \sqrt{\#\{\A_{12}\cup \I_{12}\}}}\|{\bm{\beta}}^*_{\I_1}\|_2\right)\\
\leqslant &2p \exp\{-nC_1\vartheta /s^*p_{\max}\} \leqslant \delta_1,
\end{align*}
where the second inequality follows from \eqref{lemB1}.

Next, we show \eqref{ap:Le4}. We have
\begin{align*}
  &P\left(\|{\bm{H}}_{\hat{\A}}{\bm{\varepsilon}}\|_2> \sqrt{\frac{\eta}{4}n\left(c_*(T)-\frac{\omega_T^2}{c_*(T)}\right)\|{\bm{\beta}}^*_{\I_1}\|_2^2}\right) \\
\leqslant &P\left(\frac{\sqrt{nc^*(T)}}{nc_*(T)}\|{\bm{X}}_{\hat{\A}}^\top {\bm{\varepsilon}}\|_2>\sqrt{\frac{\eta}{4}n\left(c_*(T)-\frac{\omega_T^2}{c_*(T)}\right)\|{\bm{\beta}}^*_{\I_1}\|_2^2}\right)\\
\leqslant & \sum_{j\in \hat{\A}}\sum_{i=1}^{p_j}P\left(|({\bm{X}}^{(i)}_{G_j})^\top{\bm{\varepsilon}}|>nc_*(T)\sqrt{\frac{\eta}{4c^*(T)\#\{\hat \A\}}\left(c_*(T)-\frac{\omega_T^2}{c_*(T)}\right)\|{\bm{\beta}}^*_{\I_1}\|_2^2}\right)\\
\leqslant &2p \exp\{-nC_2\vartheta/Tp_{\max}\} \leqslant \delta_2.
\end{align*} 
Here $C_2$ depends on the spectrum bounds in GSRC.
Similarly, we can obtain \eqref{ap:Le3}.
\qedsymbol
\end{proof}



\subsection{Proof of Theorem 1}
\begin{proof}
Assume $\I_1 \neq \emptyset$ and show that it will lead to a contradiction. Note that $\A^* \cap \tilde{\I} = \A_{12}\cup \I_{12}$.
Denote $\tilde{{\bm{\beta}}}$ as the least-squares estimator on $\cup_{j \in \tilde{\A}} G_j$.
The loss of $(\tilde{\A}, \tilde{\I})$ is
\begin{align}\label{thmA1}
  \begin{split}
  2nL(\tilde{{\bm{\beta}}}) = & {\bm{y}}^\top(\bm{\mathrm{I}}_n-{\bm{H}}_{\tilde{\A}}){\bm{y}} =({\bm{X}}_{\A_{12}\cup \I_{12}}{\bm{\beta}}^*_{\A_{12}\cup \I_{12}}+{\bm{\varepsilon}})^\top (\bm{\mathrm{I}}_n-{\bm{H}}_{\tilde{\A}}) ({\bm{X}}_{\A_{12}\cup \I_{12}}{\bm{\beta}}^*_{\A_{12}\cup \I_{12}}+{\bm{\varepsilon}})\\
=&({\bm{X}}_{\A_{12}\cup \I_{12}}{\bm{\beta}}^*_{\A_{12}\cup \I_{12}})^\top (\bm{\mathrm{I}}_n-{\bm{H}}_{\tilde{\A}}){\bm{X}}_{\A_{12}\cup \I_{12}}{\bm{\beta}}^*_{\A_{12}\cup \I_{12}}+{\bm{\varepsilon}}^\top (\bm{\mathrm{I}}_n-{\bm{H}}_{\tilde{\A}}){\bm{\varepsilon}}+\\
&2({\bm{X}}_{\A_{12}\cup \I_{12}}{\bm{\beta}}^*_{\A_{12}\cup \I_{12}})^\top(\bm{\mathrm{I}}_n-{\bm{H}}_{\tilde{\A}}){\bm{\varepsilon}}\\
\leqslant&nc^*(T)(\|{\bm{\beta}}^*_{\A_{12}}\|_2^2+\|{\bm{\beta}}^*_{\I_{12}}\|_2^2)+f_1({\bm{\varepsilon}})\\
\leqslant&8nc^*(T)(\frac{(1+\eta)\left(\omega_T+\frac{\omega_T^2}{c_*(T)}\right)}{c_*(T)})^2\|{\bm{\beta}}^*_{\I_1}\|_2^2+f_1({\bm{\varepsilon}})
  \end{split}
\end{align}
where the second inequality follows from \eqref{lemB1}, and denote $f_1({\bm{\varepsilon}}) = {\bm{\varepsilon}}^\top {\bm{\varepsilon}}+|{\bm{\varepsilon}}^\top {\bm{H}}_{\tilde{\A}}{\bm{\varepsilon}}|+2|({\bm{X}}_{\A_{12}\cup \I_{12}}{\bm{\beta}}^*_{\A_{12}\cup \I_{12}})^\top(\bm{\mathrm{I}}_n-{\bm{H}}_{\tilde{\A}}){\bm{\varepsilon}}|$.
The equality holds for the loss of $(\hat \A, \hat \I)$ is
\begin{align}\label{thmA2}
  \begin{split}
  2nL(\hat{{\bm{\beta}}}) = &{\bm{y}}^\top(\bm{\mathrm{I}}_n-{\bm{H}}_{\hat{\A}}){\bm{y}}=({\bm{X}}_{\I_{1}}{\bm{\beta}}^*_{\I_{1}}+{\bm{\varepsilon}})^\top (\bm{\mathrm{I}}_n-{\bm{H}}_{\hat{\A}}) ({\bm{X}}_{\I_{1}}{\bm{\beta}}^*_{\I_{1}}+{\bm{\varepsilon}})\\
\geqslant&n\left(c_*(T)-\frac{\omega_T^2}{c_*(T)}\right)\|{\bm{\beta}}^*_{\I_{1}}\|_2^2+f_2({\bm{\varepsilon}}),
\end{split}
\end{align}
where $f_2({\bm{\varepsilon}}) = {\bm{\varepsilon}}^\top {\bm{\varepsilon}}-2|({\bm{X}}_{\I_{1}}{\bm{\beta}}^*_{\I_{1}})^\top (\bm{\mathrm{I}}_n-{\bm{H}}_{\hat \A}){\bm{\varepsilon}}|-|{\bm{\varepsilon}}^\top {\bm{H}}_{\hat{\A}}{\bm{\varepsilon}}|$.

Therefore, with probability at least $1-\delta_1-\delta_2$, from \eqref{thmA1} and \eqref{thmA2}, we have
\begin{align*}
  L(\hat {\bm{\beta}}) - L(\tilde{{\bm{\beta}}})
  \geqslant&  \frac{1}{2n}( n\left(c_*(T)-\frac{\omega_T^2}{c_*(T)}\right)\|{\bm{\beta}}^*_{\I_1}\|_2^2+ f_2({\bm{\varepsilon}})-\\
  &8nc^*(T)(\frac{(1+\eta)\left(\omega_T+\frac{\omega_T^2}{c_*(T)}\right)}{c_*(T)})^2\|{\bm{\beta}}^*_{\I_1}\|_2^2  - f_1({\bm{\varepsilon}}))\\
  \geqslant&  \frac{(1-\eta)}{2}\left(c_*(T)-\frac{\omega_T^2}{c_*(T)}\right)\|{\bm{\beta}}^*_{\I_1}\|_2^2-4c^*(T)(\frac{(1+\eta)(\omega_T+\frac{\omega_T^2}{c_*(T)})}{c_*(T)})^2\|{\bm{\beta}}^*_{\I_1}\|_2^2\\
  \geqslant& \frac{(1-\mu_T)(1-\eta)\left(c_*(T)-\frac{\omega_T^2}{c_*(T)}\right)}{2}\|{\bm{\beta}}^*_{\I_1}\|_2^2\\
  \geqslant& \frac{(1-\mu_T)(1-\eta)\left(c_*(T)-\frac{\omega_T^2}{c_*(T)}\right)}{2}\vartheta\\
  >& \pi_T,
\end{align*}
where the second inequality follows from \eqref{ap:Le5}-\eqref{ap:Le3}, the third inequality follows form (C3) and the fourth inequality follows from (C4) and (C5).
Consequently,
\begin{equation*}
  P\left(L(\hat {\bm{\beta}}) - L(\tilde{{\bm{\beta}}})>\pi_T\right) \geqslant 1- \delta_1-\delta_2.
\end{equation*}
Thus Algorithm 1 continues iterations, which leads to a contradiction with $\I_1 \neq \emptyset$.
\qedsymbol
\end{proof}

\subsection{Proof of Theorem 2}

\begin{proof}
To prove Theorem 2, we need to analyze the gap between $\log L({\bm{\beta}}_1)$ and $\log L({\bm{\beta}}_2)$.
Using the inequality that $1 - \frac{1}{x} \leqslant \log x \leqslant x-1$ for any $x>0$, we have
\begin{equation}\label{thmB1}
  \dfrac{L({\bm{\beta}}_1)-L({\bm{\beta}}_2)}{L({\bm{\beta}}_1)}\leqslant \log \dfrac{L({\bm{\beta}}_1)}{L({\bm{\beta}}_2)} \leqslant \dfrac{L({\bm{\beta}}_1)-L({\bm{\beta}}_2)}{L({\bm{\beta}}_2)}.
\end{equation}
Let $\hat{{\bm{\beta}}}^* = \arg\min\limits_{ {\bm{\beta}}_{\I^*}=0}L({\bm{\beta}})$ be the least-squares estimator on $\cup_{j \in \A^*} G_j$.

First, we consider the case when $T < s^*$. With probability at least $1-\delta_1-\delta_2$, we have
\begin{align}\label{thmB5}
  \begin{split}
  2nL(\hat{{\bm{\beta}}}) - 2nL(\hat{\bm{\beta}}^*)=&{\bm{y}}^\top (\bm{\mathrm{I}}_n-{\bm{H}}_{\hat \A}){\bm{y}}-{\bm{\varepsilon}}^\top (\bm{\mathrm{I}}_n-{\bm{H}}_{\A^*}){\bm{\varepsilon}}\\
  =&({\bm{X}}_{\I_{1}}{\bm{\beta}}^*_{\I_{1}}+{\bm{\varepsilon}})^\top (\bm{\mathrm{I}}_n-{\bm{H}}_{\hat{\A}}) ({\bm{X}}_{\I_{1}}{\bm{\beta}}^*_{\I_{1}}+{\bm{\varepsilon}})-{\bm{\varepsilon}}^\top (\bm{\mathrm{I}}_n-{\bm{H}}_{\A^*}){\bm{\varepsilon}}\\
  \geqslant&n\left(c_*(T)-\frac{\omega_T^2}{c_*(T)}\right)\|{\bm{\beta}}^*_{\I_{1}}\|_2^2-2|({\bm{X}}_{\I_{1}}{\bm{\beta}}^*_{\I_{1}})^\top (\bm{\mathrm{I}}_n-{\bm{H}}_{\hat \A}){\bm{\varepsilon}}|-\\
  &|{\bm{\varepsilon}}^\top {\bm{H}}_{\hat{\A}}{\bm{\varepsilon}}|-|{\bm{\varepsilon}}^\top {\bm{H}}_{\A^*}{\bm{\varepsilon}}|\\
  \geqslant&(1-\frac{\eta}{2})n\left(c_*(T)-\frac{\omega_T^2}{c_*(T)}\right)\|{\bm{\beta}}^*_{\I_1}\|_2^2-|{\bm{\varepsilon}}^\top {\bm{H}}_{\A^*}{\bm{\varepsilon}}|,
  \end{split}
\end{align}
where the last inequality follows from \eqref{ap:Le5}, \eqref{ap:Le4} and \eqref{thmA2}.
Following from \eqref{ap:Le4} in Lemma \ref{lem:B}, we have
\begin{align}\label{thmB2}
  P\left(|{\bm{\varepsilon}}^\top {\bm{H}}_{\A^*} {\bm{\varepsilon}}| > \frac{n\eta}{2}(c_*(T)-\frac{\omega_T^2}{c_*(T)})\|{\bm{\beta}}^*_{\I_1}\|_2^2 \right) \leqslant \delta_2.
\end{align}
Combine \eqref{thmB5} and \eqref{thmB2}, with probability at least $1-\delta_1-\delta_2$,
\begin{align}\label{eq:a1}
  2nL(\hat{{\bm{\beta}}}) - 2nL(\hat{\bm{\beta}}^*)\geqslant(1-\eta)n\left(c_*(T)-\frac{\omega_T^2}{c_*(T)}\right)\|{\bm{\beta}}^*_{\I_1}\|_2^2.
\end{align}
Now turn to $2nL(\hat{{\bm{\beta}}}^*)$. With probability at least $1-\delta_2$ we have
\begin{align}\label{eq:a2}
  \begin{split}
  2nL(\hat{{\bm{\beta}}}^*) =& {\bm{\varepsilon}}^\top (\bm{\mathrm{I}}_n-{\bm{H}}_{\A^*}){\bm{\varepsilon}}\\
  \leqslant&\|{\bm{\varepsilon}}\|_2^2 + |{\bm{\varepsilon}}^\top {\bm{H}}_{\A^*} {\bm{\varepsilon}}|\\
  \leqslant& 2nL({\bm{\beta}}^*)+\frac{n\eta}{2}\left(c_*(T)-\frac{\omega_T^2}{c_*(T)}\right)\|{\bm{\beta}}^*_{\I_1}\|_2^2.
\end{split}
\end{align}
Combine \eqref{thmB1}, \eqref{eq:a1} and \eqref{eq:a2}, with probability at least $1-\delta_1-\delta_2$,
\begin{align*}
  \log\dfrac{L(\hat{{\bm{\beta}}})}{L(\hat{{\bm{\beta}}}^*)} \geqslant \frac{(1-\eta)(c_*(T)-\frac{\omega_T^2}{c_*(T)})\|{\bm{\beta}}_{\I_1}^*\|_2^2}{2L({\bm{\beta}}^*)+\frac{\eta}{2}(c_*(T)-\frac{\omega_T^2}{c_*(T)})\|{\bm{\beta}}^*_{\I_1}\|_2^2}=O(1).
\end{align*}
Let $\delta = O(p^{-\alpha})$.
Note that with probability at least $1-\delta_1-\delta_2 \geqslant 1 - \delta$ for some constant $0 < \alpha < 1$. 
Consequently, with probability at least $1-O(p^{-\alpha})$, 
\vspace{-0.1cm}
\begin{align*}
  \text{GIC}(\hat\A)-\text{GIC}(\A^*) &= n\log \dfrac{L(\hat{{\bm{\beta}}})}{L(\hat {\bm{\beta}}^*)}-(\#\{\A^*\}-\#\{\hat \A\})\log J\log(\log n)\\
  &\geqslant O(n) - \#\{\A^*\}\log J\log(\log n)\\
  &\geqslant O(n) - o(n) >0
\end{align*}
for a sufficiently large $n$, where the third inequality follows from condition (C6).

On the other hand, from Theorem 1, it holds for $T \geqslant s^*$ that
\begin{equation*}
  P\left(\hat \A \supseteq \A^*\right) \geqslant 1-\delta_1-\delta_2 \geqslant 1- \delta.
\end{equation*}
Therefore, when $T \geqslant s^*$, with probability at least $1-O(p^{-\alpha})$, we have
\begin{align*}
  L(\hat{{\bm{\beta}}}) =\frac{1}{2n} {\bm{y}}^\top (\bm{\mathrm{I}}_n-{\bm{H}}_{\hat \A}){\bm{y}} = \frac{1}{2n}{\bm{\varepsilon}}^\top (\bm{\mathrm{I}}_n-{\bm{H}}_{\hat \A}){\bm{\varepsilon}}.
\end{align*}
Especially, when $T=s^*$, we have
\begin{align*}
  L(\hat{{\bm{\beta}}})=L(\hat{{\bm{\beta}}^*}) = \frac{1}{2n}{\bm{\varepsilon}}^\top (\bm{\mathrm{I}}_n-{\bm{H}}_{\A^*}){\bm{\varepsilon}}.
\end{align*}
 
Let $\hat \A = \A^* \cup \mB$,
\vspace{-0.1cm}
\begin{align*}
  L(\hat{{\bm{\beta}}}^*)-L(\hat{{\bm{\beta}}}) &= \dfrac{1}{2n}{\bm{\varepsilon}}^\top({\bm{H}}_{\hat\A}-{\bm{H}}_{\A^*}){\bm{\varepsilon}}=\dfrac{1}{2n}{\bm{\varepsilon}}^\top(\bm{\mathrm{I}} - {\bm{H}}_{\A^*}){\bm{H}}_{\hat\A}(\bm{\mathrm{I}} - {\bm{H}}_{\A^*}){\bm{\varepsilon}}\\
  &=\dfrac{1}{2n}{\bm{\varepsilon}}^\top(\bm{\mathrm{I}}_n-{\bm{H}}_{\A^*}){\bm{X}}_{\mB}({\bm{X}}_{\mB}^\top(\bm{\mathrm{I}}_n-{\bm{H}}_{\A^*}){\bm{X}}_{\mB})^{-1}{\bm{X}}_{\mB}^\top(\bm{\mathrm{I}}_n-{\bm{H}}_{\A^*}){\bm{\varepsilon}}\\
  &=\dfrac{1}{2n}\|({\bm{X}}_{\mB}^\top(\bm{\mathrm{I}}_n-{\bm{H}}_{\A^*}){\bm{X}}_{\mB})^{-\frac{1}{2}}{\bm{X}}_{\mB}^\top(\bm{\mathrm{I}}_n-{\bm{H}}_{\A^*}){\bm{\varepsilon}}\|_2^2.
\end{align*}
Note that
\begin{align}\label{thmB3}
  \begin{split}
  &P\left(\frac{1}{\sqrt{2n}}\|({\bm{X}}_{\mB}^\top(\bm{\mathrm{I}}_n-{\bm{H}}_{\A^*}){\bm{X}}_{\mB})^{-\frac{1}{2}}{\bm{X}}_{\mB}^\top(\bm{\mathrm{I}}_n-{\bm{H}}_{\A^*}){\bm{\varepsilon}}\|_2 \geqslant t\right) \\
  \leqslant &P\left(\|{\bm{X}}_{\mB}^\top(\bm{\mathrm{I}}_n-{\bm{H}}_{\A^*}){\bm{\varepsilon}}\|_2 \geqslant \sqrt{2(c_*(T)-\frac{\omega_T^2}{c_*(T)})}nt\right)\\
  \leqslant & \sum_{j\in \mB}\sum_{i=1}^{p_j}P\left(|({\bm{X}}^{(i)}_{G_j})^\top {\bm{\varepsilon}}|>\sqrt{\frac{2}{\#\{\mB\}}(c_*(T)-\frac{\omega_T^2}{c_*(T)})}nt\right)\\
  \leqslant &2p \exp\{-\dfrac{C_1nt^2}{\#\{\mB\}}\}=\delta
\end{split}
\end{align}
for some positive constant $C_1$ depending on the spectrum bounds in GSRC.
Given $\delta = O(p^{-\alpha})$, with probability at least $1-\delta$, we can calculate the corresponding $t$ following from
\begin{align*}
  \frac{1}{2n}\|({\bm{X}}_{\mB}^\top(\bm{\mathrm{I}}_n-{\bm{H}}_{\A^*}){\bm{X}}_{\mB})^{-\frac{1}{2}}{\bm{X}}_{\mB}^\top(\bm{\mathrm{I}}_n-{\bm{H}}_{\A^*}){\bm{\varepsilon}}\|_2^2\leqslant \frac{\#\{\mB\}}{nC_1} \log \frac{2p}{\delta}.
\end{align*}
Therefore, with probability at least $1-\delta$,
\begin{align}\label{eq:a3}
  \begin{split}
  L(\hat{{\bm{\beta}}}^*)-L(\hat{{\bm{\beta}}}) =& \frac{1}{2n}\|({\bm{X}}_{\mB}^\top(\bm{\mathrm{I}}_n-{\bm{H}}_{\A^*}){\bm{X}}_{\mB})^{-\frac{1}{2}}{\bm{X}}_{\mB}^\top(\bm{\mathrm{I}}_n-{\bm{H}}_{\A^*}){\bm{\varepsilon}}\|_2^2\\
  \leqslant& \dfrac{(1+\alpha)\#\{\mB\}\log p}{nC_1}.
\end{split}
\end{align}
Define $\hat\sigma$ as the standard deviation of random variable ${\bm{\varepsilon}}_i$.
Similar to \eqref{eq:a3}, with probability at least $1-O(p^{-\alpha})$,  we have
\begin{align}\label{eq:a4}
  \begin{split}
  L(\hat{{\bm{\beta}}}) &=\frac{1}{2n}{\bm{\varepsilon}}^\top (\bm{\mathrm{I}}_n-{\bm{H}}_{\hat{\A}}) {\bm{\varepsilon}}\\
  &\geqslant \dfrac{1}{2n}\|{\bm{\varepsilon}}\|_2^2 - \frac{1}{2n}{\bm{\varepsilon}}^\top {\bm{H}}_{\hat{\A}} {\bm{\varepsilon}}\\
  &\geqslant \dfrac{\hat\sigma^2}{2} - \dfrac{(1+\alpha)\#\{\hat \A\}\log p}{nC_2},
  \end{split}
\end{align}
for some positive constant $C_2$ depending on the spectrum bounds in GSRC, where the second inequality follows from the law of large number.
From (C6), we have
\begin{align}\label{thmB4}
\frac{(1+\alpha)\#\{\hat \A\}\log p}{nC_2} \leqslant \frac{(1+\alpha)\#\{\hat \A_{T_{\max}}\}\log p}{nC_2}\rightarrow 0
\end{align}
for a sufficiently large $n$.
Combining \eqref{thmB1} \eqref{eq:a3}, \eqref{eq:a4} and \eqref{thmB4}, we have
\begin{equation*}
  n\log \dfrac{L(\hat{{\bm{\beta}}}^*)}{L(\hat{{\bm{\beta}}})}\leqslant \frac{\frac{\#\{\mB\}\log(p^{1+\alpha})}{C_1}}{\frac{\hat\sigma^2}{2} - \frac{\#\{\hat \A\}\log(p^{1+\alpha})}{nC_2}}\rightarrow \frac{2(1+\alpha)\#\{\mB\}\log p}{\hat\sigma^2C_1}
\end{equation*}
for a sufficiently large $n$.
Consequently, following from (C7), we have
\begin{align*}
  \text{GIC}(\A^*)-\text{GIC}(\hat \A) &= n\log\dfrac{L(\hat{{\bm{\beta}}}^*)}{L(\hat{{\bm{\beta}}})}-\#\{\mB\}\log J\log(\log n)\\
  &\leqslant O\left(\#\{\mB\}\log p \right)-\#\{\mB\}\log J\log(\log n)\\
  &\leqslant O\left(\#\{\mB\}\log (Jp_{\max})\right) -\#\{\mB\}\log J\log(\log n)\\
  &= O\left(\#\{\mB\}\log J +  \#\{\mB\}\log p_{\max}\right)-\#\{\mB\}\log J\log(\log n)\\
  &< 0
\end{align*}
for a sufficiently large $n$.
Therefore, Algorithm 2 identifies the true subset of groups $\A^*$ with probability at least $1-O(p^{-\alpha})$.
\qedsymbol
\end{proof}

\subsection{Proof of Theorem 3}
Denote $\mS_1^k$ and $\mS_2^k$ as the exchange subsets of groups in the $k$th iteration and
$\A_1^k = \A^k \cap \A^*,\ \I_1^k = \I^k \cap \A^*, \A_{12}^k = \A_1^k \cap \mS_1^k,\ \I_{12}^k = \I_1^k \cap (\mS_2^k)^c.$
\begin{proof}
Note that $\I_1^{k+1} = \A_{12}^k \cup \I_{12}^k$.
The error of loss in the $(k+1)$th iteration is
\begin{align}\label{thmC1}
  \begin{split}
  |2nL({\bm{\beta}}^{k+1}) - 2nL({\bm{\beta}}^*)|
  = & |({\bm{X}}_{\I_{1}^{k+1}}{\bm{\beta}}^*_{\I_{1}^{k+1}}+{\bm{\varepsilon}})^\top(\bm{\mathrm{I}}_n-{\bm{H}}_{ \A^{k+1}})({\bm{X}}_{\I_{1}^{k+1}}{\bm{\beta}}^*_{\I_{1}^{k+1}}+{\bm{\varepsilon}})-{\bm{\varepsilon}}^\top{\bm{\varepsilon}}|\\
  \leqslant&|({\bm{X}}_{\A_{12}^k\cup \I_{12}^k}{\bm{\beta}}^*_{\A_{12}^k\cup \I_{12}^k})^\top (\bm{\mathrm{I}}_n-{\bm{H}}_{\A^{k+1}}) ({\bm{X}}_{\A_{12}^k\cup \I_{12}^k}{\bm{\beta}}^*_{\A_{12}^k\cup \I_{12}^k})|+\\
  &2|({\bm{X}}_{\A_{12}^k\cup \I_{12}^k}{\bm{\beta}}^*_{\A_{12}^k\cup \I_{12}^k})^\top(\bm{\mathrm{I}}_n-{\bm{H}}_{ \A^{k+1}}){\bm{\varepsilon}}|+|{\bm{\varepsilon}}^\top {\bm{H}}_{ \A^{k+1}}{\bm{\varepsilon}}|\\
\leqslant & 8nc^*(T)[\frac{(1+\eta)(\omega_T+\frac{\omega_T^2}{c_*(T)})}{c_*(T)}]^2\|{\bm{\beta}}^*_{\I_1^{k}}\|_2^2+h_1({\bm{\varepsilon}}),
\end{split}
\end{align}
where $h_1({\bm{\varepsilon}}) = 2|({\bm{X}}_{\A_{12}^k\cup \I_{12}^k}{\bm{\beta}}^*_{\A_{12}^k\cup \I_{12}^k})^\top (\bm{\mathrm{I}}_n-{\bm{H}}_{ \A^{k+1}}){\bm{\varepsilon}}|+|{\bm{\varepsilon}}^\top {\bm{H}}_{ \A^{k+1}}{\bm{\varepsilon}}|$.

Similarly, the error of loss in the $k$th iteration is
\begin{align}\label{thmC2}
  \begin{split}
  |2nL({\bm{\beta}}^{k}) - 2nL({\bm{\beta}}^*)|=&|({\bm{X}}_{\I_{1}^{k}}{\bm{\beta}}^*_{\I_{1}^{k}}+{\bm{\varepsilon}})^\top(\bm{\mathrm{I}}_n-{\bm{H}}_{ \A^{k}})({\bm{X}}_{\I_{1}^{k}}{\bm{\beta}}^*_{\I_{1}^{k}}+{\bm{\varepsilon}})-{\bm{\varepsilon}}^\top{\bm{\varepsilon}}|\\
  \geqslant& |({\bm{X}}_{\I_{1}^{k}}{\bm{\beta}}^*_{\I_{1}^{k}})^\top(\bm{\mathrm{I}}_n-{\bm{H}}_{ \A^{k}})({\bm{X}}_{\I_{1}^{k}}{\bm{\beta}}^*_{\I_{1}^{k}})|-\\
&2|({\bm{X}}_{\I_{1}^{k}}{\bm{\beta}}^*_{\I_{1}^{k}})^\top(\bm{\mathrm{I}}_n-{\bm{H}}_{ \A^{k}}){\bm{\varepsilon}}|-|{\bm{\varepsilon}}^\top {\bm{H}}_{ \A^{k}}{\bm{\varepsilon}}|\\
\geqslant&n(c^*(T)-\frac{\omega_T^2}{c_*(T)})\|{\bm{\beta}}^*_{\I_{1}^k}\|_2^2-h_2({\bm{\varepsilon}}),
\end{split}
\end{align}
where $h_2({\bm{\varepsilon}}) = 2|({\bm{X}}_{\I_{1}^{k}}{\bm{\beta}}^*_{\I_{1}^{k}})^\top(\bm{\mathrm{I}}_n-{\bm{H}}_{ \A^{k}}){\bm{\varepsilon}}|+|{\bm{\varepsilon}}^\top {\bm{H}}_{ \A^{k}}{\bm{\varepsilon}}|$.
From the proof of Lemma~\ref{lem:B}, we have
\begin{equation}\label{thmC3}
  h_1({\bm{\varepsilon}}) \leqslant \frac{\eta}{2}\mu_T n\left(c^*(T)-\frac{\omega_T^2}{c_*(T)}\right)\|{\bm{\beta}}^*_{\I_{1}^k}\|_2^2
\end{equation}
and
\begin{equation}\label{thmC4}
  h_2({\bm{\varepsilon}}) \leqslant \frac{\eta}{2}n\left(c^*(T)-\frac{\omega_T^2}{c_*(T)}\right)\|{\bm{\beta}}^*_{\I_{1}^k}\|_2^2
\end{equation}
with probability at least $1-\delta_1-\delta_2$.
Combine \eqref{thmC1}-\eqref{thmC4},
\begin{align}\label{thmC5}
  \begin{split}
  |2nL({\bm{\beta}}^{k+1}) - 2nL({\bm{\beta}}^*)|\leqslant &8nc^*(T)(\frac{(1+\eta)(\omega_T+\frac{\omega_T^2}{c_*(T)})}{c_*(T)})^2\|{\bm{\beta}}^*_{\I_1^k}\|_2^2+h_1({\bm{\varepsilon}})\\
  \leqslant& \mu_T(1-\eta)n\left(c^*(T)-\frac{\omega_T^2}{c_*(T)}\right)\|{\bm{\beta}}^*_{\I_{1}^k}\|_2^2+h_1({\bm{\varepsilon}})\\
  \leqslant& \mu_T|2nL({\bm{\beta}}^{k}) - 2nL({\bm{\beta}}^*)|.
\end{split}
\end{align}
Let $\A^0 = \emptyset$ and use \eqref{thmC5} repeatedly,
\begin{align}\label{thmC6}
  \begin{split}
  |2nL({\bm{\beta}}^{k+1}) - 2nL({\bm{\beta}}^*)|\leqslant& \mu_T |2nL({\bm{\beta}}^{k+1}) - 2nL({\bm{\beta}}^*)|\\
  \leqslant& \mu_T^{k+1} |2nL({\bm{\beta}}^0) - 2nL({\bm{\beta}}^*)| \\
  \leqslant& \mu_T^{k+1}\|{\bm{y}}\|_2^2.
\end{split}
\end{align}
This completes the proof of part $(i)$, which shows the error bounds of loss decay geometrically.

Next, we prove the lower bound of the error of loss when $\I^k_1 \neq \emptyset$.
From \eqref{thmC2} and \eqref{thmC4}, with probability at least $1-\delta_1-\delta_2$, we have
\begin{align}\label{thmC7}
  \begin{split}
  |2nL({\bm{\beta}}^{k}) - 2nL({\bm{\beta}}^*)| \geqslant& n\left(c_*(T)-\frac{\omega_T^2}{c_*(T)}\right)\|{\bm{\beta}}^*_{\I_{1}^k}\|_2^2-h_2({\bm{\varepsilon}})\\
  \geqslant& (1-\frac{\eta}{2})n\left(c_*(T)-\frac{\omega_T^2}{c_*(T)}\right)\vartheta >0.
\end{split}
\end{align}
When the lower bounds exceed the upper bounds in \eqref{thmC6} , we can conclude that $\A^k \supseteq \A^*$.
Therefore, we have $\A^k \supseteq \A^*$ if
\begin{equation*}
  |2nL({\bm{\beta}}^{k}) - 2nL({\bm{\beta}}^*)| \leqslant \mu_T^k \|{\bm{y}}\|_2^2 \leqslant (1-\frac{\eta}{2})n\left(c_*(T)-\frac{\omega_T^2}{c_*(T)}\right)\vartheta
\end{equation*}
holds, equivalently,
\begin{equation*}
  \A^k \supseteq \A^*,\quad k>\log_{\frac{1}{\mu_T}}\dfrac{\|{\bm{y}}\|_2^2}{(1-\frac{\eta}{2})n\left(c_*(T)-\frac{\omega_T^2}{c_*(T)}\right)\vartheta},
\end{equation*}
which completes the proof of Theorem 3.
\qedsymbol
\end{proof}

\subsection{Proof of Corollary 3}
\begin{proof}
  Assume $\A^k \supseteq \A^*$. We have $\A^{k+1} \supseteq \A^*$ with probability at least $1-\delta_1-\delta_2$ from Theorem 3.
  Therefore, we have $L({\bm{\beta}}^{k+1}) = \frac{1}{2n}{\bm{\varepsilon}}^\top (\bm{\mathrm{I}}_n-{\bm{H}}_{\A^{k+1}}) {\bm{\varepsilon}}$ and $L({\bm{\beta}}^k) = \frac{1}{2n}{\bm{\varepsilon}}^\top (\bm{\mathrm{I}}_n-{\bm{H}}_{\A^k}) {\bm{\varepsilon}}$.

  Following similar derivation in \eqref{eq:a4}, with probability at least $1-O(p^{-\alpha})$, we have
  \begin{equation*}
      L({\bm{\beta}}^k) \leqslant \dfrac{\hat\sigma^2}{2} + \dfrac{\#\{\A^k\}\log(p^{1+\alpha})}{nC}
  \end{equation*}
  and
  \begin{equation*}
      L({\bm{\beta}}^{k+1}) \geqslant \dfrac{\hat\sigma^2}{2} - \dfrac{\#\{\A^{k+1}\}\log(p^{1+\alpha})}{nC},
  \end{equation*}
for some positive constant $C$ depending on the spectrum bounds in GSRC.
From (C5),
\begin{equation*}
  L({\bm{\beta}}^k)-L({\bm{\beta}}^{k+1}) \leqslant \dfrac{2\#\{\A^{k+1}\}\log(p^{1+\alpha})}{nC}  \leqslant \dfrac{2(1+\alpha)Tp_{\max}\log p}{nC} \leqslant \pi_T.
\end{equation*}

The gap of loss is smaller than the threshold $\pi_T$ in Algorithm 1 after $k$th iteration.
Combining Theorem 3, we can conclude that with probability at least $1-O(p^{-\alpha})$, Algorithm 1 stops after $O\left(\log_{\frac{1}{\mu_T}}\dfrac{\|{\bm{y}}\|_2^2}{(1-\frac{\eta}{2})n(c_*(T)-\frac{\omega_T^2}{c_*(T)})\vartheta}\right)$ iterations when $T>s$.
\qedsymbol
\end{proof}

\subsection{Proof of Theorem 4}
\begin{proof}
First, consider $0 < T < s^*$.
Since the loss decreases at least $\pi_T$ in each iteration, Algorithm 1 stops after $O(\frac{\|\bm y\|_2^2}{\pi_T})$ iterations for a given $T$.
Next consider $s^* \leqslant T \leqslant T_{\max}$. By Corollary 3 and (C3), Algorithm 1 stops after $O\left(\log_{\frac{1}{\mu_T}} \frac{\|{\bm{y}}\|_2^2}{T p_{\max} \log p\log(\log n)}\right)$ iterations.

Now we analyze the computational complexity of Algorithm 1 for a given model size $T$. First, computing the primal variable and dual variable takes $O(nTp_{\max}+np)$ steps and computing the $\ell_2$ norm of each group takes $O(p)$ steps. Next, finding the smallest or largest $C_{\max}$ contributions takes $O(JC_{\max})$ steps via Hoare's selection algorithm \citep{hoare}. 
For group splicing operations, the exchange repeats at most $C_{\max}$ times. Thus $O\left((nTp_{\max}+np)C_{\max}\right)$ steps at most are demanded.
Therefore, the total computational complexity of Algorithm 1 is
\begin{equation*}
  O\left((\log_{\frac{1}{\mu_T}}\frac{\|{\bm{y}}\|_2^2}{Tp_{\max}\log p\log(\log n)} I(s^* \leqslant T)+\frac{\|{\bm{y}}\|_2^2}{\pi_T}I(s^*> T))\times ((nTp_{\max}+np+J)C_{\max})\right).
\end{equation*}

Since $T$ varies from $1$ to $T_{\max}$ in Algorithm 2, the total computational complexity is
\begin{align*}
  \begin{split}
    O&(\log_{\frac{1}{\mu_T}}\frac{\|{\bm{y}}\|_2^2}{p_{\max}\log p\log(\log n)}((nT_{\max}p+nT_{\max}^2 p_{\max}+J)C_{\max})+\\
    &\frac{n\|{\bm{y}}\|_2^2}{p_{\max}\log p\log(\log n)}((np+ns^* p_{\max}+J)C_{\max}))\\
    \leqslant O&\left((T_{\max}\log_{\frac{1}{\mu_T}}\frac{\|{\bm{y}}\|_2^2}{p_{\max}\log p\log(\log n)}+\frac{n\|{\bm{y}}\|_2^2}{p_{\max}\log p\log(\log n)})(np+nT_{\max}p_{\max}+J)C_{\max}\right).
  \end{split}
\end{align*}
  \qedsymbol
\end{proof}

\section{$\ell_2$ error bounds}\label{app-sec:l2_bound}
Here we consider the $\ell_2$ error bounds of the estimator of GSplicing, which can help us understand how the estimator gradually approaches the ground truth.
\begin{theorem}\label{thm:l2_error}
Assume Conditions (C1)-(C5) hold, when $T \geqslant s^*$, we have
  \begin{equation*}
    P \left( \|{\bm{\beta}}^k - {\bm{\beta}}^*\|_2^2 \leqslant \frac{1+\eta+\frac{\omega_T}{c_*(T)}}{(1-\frac{\eta}{2})n\left(c_*(T)-\frac{\omega_T^2}{c_*(T)}\right)}\mu_T^k\|{\bm{y}}\|_2^2  \right) \geqslant 1 - \delta_1-\delta_2.
  \end{equation*}
\end{theorem}
Theorem~\ref{thm:l2_error} shows that, with high probability,
the $\ell_2$ error of estimator can be bounded by
a term proportional to $\|{\bm{y}}\|_2^2$.
Furthermore, an interesting fact unveiled by Theorem~\ref{thm:l2_error} is that
the error bounds decay geometrically during iterations.
Additionally, the error bounds of prediction, an immediate corollary of Theorem~\ref{thm:l2_error}, are given in the following.
\begin{corollary}\label{coro:predition_error}
  Assume the conditions in Theorem~\ref{thm:l2_error} hold. Prediction error in the $k$th iteration satisfies:
  \begin{equation*}
    P\left( \|{\bm{X}}({\bm{\beta}}^k - {\bm{\beta}}^*)\|_2^2 \leqslant \frac{1+\eta+\frac{\omega_T}{c_*(T)}}{(1-\frac{\eta}{2})\left(c_*(T)-\frac{\omega_T^2}{c_*(T)}\right)}c^*(T)\mu_T^k\|{\bm{y}}\|_2^2 \right) \geqslant 1 - \delta_1 - \delta_2.
  \end{equation*}
\end{corollary}

\subsection{Proof of Theorem 6}
\begin{proof}
Note that $P\left(\|({\bm{X}}_{\A^k}^\top {\bm{X}}_{\A^k})^{-1}{\bm{X}}_{\A^k}^\top {\bm{\varepsilon}}\|_2 \leqslant \sqrt{\eta}\|{\bm{\beta}}_{\I_1^k}^*\|_2 \right) \geqslant 1- \delta_2$.
We have
\begin{align}\label{eq:a6}
  \begin{split}
  \|{\bm{\beta}}^k-{\bm{\beta}}^*\|_2^2 =&\|{\bm{\beta}}^k_{\A^k}-{\bm{\beta}}^*_{\A^k}\|_2^2+\|{\bm{\beta}}^*_{\I_1^k}\|_2^2=\|({\bm{X}}_{\A^k}^\top {\bm{X}}_{\A^k})^{-1}{\bm{X}}_{\A^k}^\top {\bm{y}}-{\bm{\beta}}^*_{\A^k}\|_2^2+\|{\bm{\beta}}^*_{\I_1^k}\|_2^2\\
=&\|({\bm{X}}_{\A^k}^\top {\bm{X}}_{\A^k})^{-1}{\bm{X}}_{\A^k}^\top ({\bm{X}}_{\A^k}{\bm{\beta}}^*_{\A^k}+{\bm{X}}_{\I_1^k}{\bm{\beta}}^*_{\I_1^k}+{\bm{\varepsilon}})-{\bm{\beta}}^*_{\A^k}\|_2^2+\|{\bm{\beta}}^*_{\I_1^k}\|_2^2\\
\leqslant & \left(1+\eta+\frac{\omega_T}{c_*(T)}\right)\|{\bm{\beta}}^*_{\I_1^k}\|_2^2.
\end{split}
\end{align}
Combine \eqref{thmC6}, \eqref{thmC7} and \eqref{eq:a6},
\begin{align*}
  \|{\bm{\beta}}^k-{\bm{\beta}}^*\|_2^2 \leqslant& \frac{1+\eta+\frac{\omega_T}{c_*(T)}}{(1-\frac{\eta}{2})n\left(c_*(T)-\frac{\omega_T^2}{c_*(T)}\right)}|2nL({\bm{\beta}}^{k}) - 2nL({\bm{\beta}}^*)|\\
  \leqslant& \frac{1+\eta+\frac{\omega_T}{c_*(T)}}{(1-\frac{\eta}{2})n\left(c_*(T)-\frac{\omega_T^2}{c_*(T)}\right)}\mu_T^k \|{\bm{y}}\|_2^2.
\end{align*}
This completes the proof of Theorem 6.
\end{proof}

\bibliographystyle{unsrtnat}
\bibliography{ref}

\end{document}